\documentclass[10pt,journal,compsoc]{IEEEtran}

\ifCLASSOPTIONcompsoc
  \usepackage[nocompress]{cite}
\else
  \usepackage{cite}
\fi
\usepackage{algorithm}
\usepackage{algorithmic}
\usepackage{amssymb,amsfonts,amsmath,amsthm}
\usepackage{multicol,multirow}
\usepackage{graphicx}

\newtheorem{theorem}{Theorem}
\newtheorem{remark}{Remark}

\newtheorem{lemma}{Lemma}
\newtheorem{definition}{Definition}
\newtheorem{assumption}{Assumption}
\ifCLASSINFOpdf
\else
\fi
\begin{document}
%
\title{Learning Rates for Nonconvex Pairwise Learning}
%
%
%
%

\author{Shaojie Li \quad Yong Liu\IEEEauthorrefmark{2}
\IEEEcompsocitemizethanks{\IEEEcompsocthanksitem  \IEEEauthorrefmark{2}Corresponding Author.
\IEEEcompsocthanksitem S. Li and Y. Liu are with the Gaoling School of Artificial Intelligence, Renmin University of China, Beijing 100872, China (email: 2020000277@ruc.edu.cn and liuyonggsai@ruc.edu.cn).
}
\thanks{Manuscript received November 09, 2021.}}

%
%

\markboth{Journal of \LaTeX\ Class Files
}%
{Shaojie Li: Learning Rates of Nonconvex Pairwise Learning}
%



\IEEEtitleabstractindextext{%
\begin{abstract}
Pairwise learning is receiving increasing attention since it covers many important machine learning tasks, e.g., metric learning, AUC maximization, and ranking. Investigating the generalization behavior of pairwise learning is thus of significance. However, existing generalization analysis mainly focuses on the convex objective functions, leaving the nonconvex learning far less explored. Moreover, the current learning rates derived for generalization performance of pairwise learning are mostly of slower order. Motivated by these problems, we study the generalization performance of nonconvex pairwise learning and provide improved learning rates. Specifically, we develop different uniform convergence of gradients for pairwise learning under different assumptions, based on which we analyze empirical risk minimizer, gradient descent, and stochastic gradient descent pairwise learning. We first successfully establish learning rates for these algorithms in a general nonconvex setting, where the analysis sheds insights on the trade-off between optimization and generalization and the role of early-stopping. We then investigate the generalization performance of nonconvex learning with a gradient dominance curvature condition. In this setting, we derive faster learning rates of order $\mathcal{O}(1/n)$, where $n$ is the sample size. Provided that the optimal population risk is small, we further improve the learning rates to $\mathcal{O}(1/n^2)$, which, to the best of our knowledge, are the first $\mathcal{O}(1/n^2)$-type of rates for pairwise learning, no matter of convex or nonconvex learning. Overall, we systematically analyzed the generalization performance of nonconvex pairwise learning.
\end{abstract}

}

\maketitle

\IEEEdisplaynontitleabstractindextext

%
\IEEEpeerreviewmaketitle

\ifCLASSOPTIONcompsoc
\IEEEraisesectionheading{\section{Introduction}\label{sec:introduction}}
\else
\section{Introduction}
\label{sec:introduction}
\fi

%
%
%
%
\IEEEPARstart{P}{airwise} 
learning focuses on learning tasks with loss functions depending on a pair of training examples, and thus has great advantage in modeling relative relationships between paired samples. As an important field of modern machine learning, pairwise learning instantiates many well-known learning tasks, for instance, similarity and metric learning \cite{jin2009regularized,cao2016generalization,li2021sharper,liu2021refined}, AUC maximization \cite{cortes2003auc,zhao2011online,gao2013one,ying2016stochastic,liu2018fast1,lei2021stochastic,yang2021learning,wang2021differentially,dang2021large}, bipartite ranking \cite{agarwal2009generalization,clemencon2008ranking,clemencon2005ranking,liu2017generalization}, gradient learning \cite{mukherjee2006learning,mukherjee2006estimation}, minimum error entropy principle \cite{hu2013learning,guo2020distributed}, multiple kernel learning \cite{kumar2012a}, and preference learning \cite{furnkranz2010preference}, etc. 

Since its significance, there has been an increasing interest in the generalization performance analysis of pairwise learning to understand why it performs well in practice. Generalization analysis investigates how the predictive models learned from training samples behave on the testing samples, which is one of the primary interests in the machine learning community \cite{krishnapuram2005sparse,liu2017algorithm,li2021orthogonal,bian2014asymptotic,xu2020upper}.
In contrast to the classical pointwise learning problems where the loss function involves single instances, pairwise learning loss concerns pairs of training samples. This coupled construction leads to the fact that the empirical risk of pairwise learning has $\mathcal{O}(n^2)$ dependent terms if there are $n$ training samples \cite{lei2021generalization}. The fundamental assumption of independent and identical distributed (i.i.d.) random variables is thus violated for the empirical risk of pairwise learning, which, unfortunately, renders the standard generalization analysis in the i.i.d. case not applicable in this context.

There are many existing studies on the generalization performance of pairwise learning, but they have the following limitations.
Firstly, they typically are not general. Specifically,
most of the existing work studies specific instantiations, for instance, bipartite ranking or AUC maximization \cite{lei2018generalization}. On the contrary, there is far less work studying the generalization performance under the framework of pairwise learning \cite{lei2020sharper,lei2021generalization}. Secondly, they mostly require convexity conditions \cite{lei2021generalization}. In the related work of learning the unified pairwise framework,
\cite{kar2013on,wang2012generalization,lin2017online} investigate the online pairwise learning, which is different from the offline setting of this paper, while \cite{papa2015sgd,shen2020stability} study the variants of stochastic gradient descent (SGD). The most related work to this paper is \cite{lei2018generalization,lei2020sharper,lei2021generalization}. In \cite{lei2018generalization}, the authors study the generalization performance of regularized empirical risk minimizer (RRM) via a peeling technology in uniform convergence. In \cite{lei2020sharper}, the authors establish the relationship between the generalization measure and algorithmic stability, and then use this connection to study the generalization performance of RRM and SGD. While in \cite{lei2021generalization}, the authors initialize a systematic generalization analysis of SGD under milder assumptions via algorithmic stability and uniform convergence of gradients.
However, the above works \cite{lei2020sharper,kar2013on,shen2020stability,papa2015sgd,lei2021generalization,lin2017online,lei2018generalization} are almost limited to the convex learning, and even often requiring the restrictive strong convexity condition. An exception is \cite{lei2021generalization}, where nonconvex learning is involved. Thirdly, in \cite{lei2021generalization}, the authors only investigate the SGD, where there are two learning rates derived for nonconvex pairwise learning. One is of order $\mathcal{O}(\sqrt{d/n})$, provided with high probability under general nonconvex assumptions, while another is of order $\mathcal{O}(n^{-\frac{2}{3}})$, provided in expectation under an extra gradient dominated assumption \cite{lei2021generalization}, where $n$ is the sample size and $d$ is the dimension of parameter space. However, one can see that these rates are of slower order.

Motivated by these problems, our goal is to provide a systematic and improved generalization analysis for nonconvex pairwise learning. 
Our contributions are summarized as follows.
\begin{itemize}
\item
We study the generalization performance of the rarely explored nonconvex pairwise learning problems. 
Our analysis is performed on the general pairwise learning framework and spans empirical risk minimizer (ERM), gradient descent (GD), and stochastic gradient descent (SGD).

\item We first consider the general nonconvex learning and provide learning rates for these algorithms. Our analysis reveals that the optimization and generalization should be balanced for achieving good learning rates, which sheds insights on the role of early-stopping. The derived rates are based on our developed uniform convergences of gradients for pairwise learning, which may be of independent interest.

\item
We then study the nonconvex learning with a commonly used curvature condition, i.e., the gradient dominance assumption. We establish faster $\mathcal{O}(1/n)$ order learning rates. If the optimal population risk is small, we further improve this learning rate to $\mathcal{O}(1/n^2)$. To our best knowledge, the $\mathcal{O}(1/n)$ rate is the first for nonconvex pairwise learning, and the $\mathcal{O}(1/n^2)$ rate is the first for pairwise learning, whether in convex learning or nonconvex learning.
\end{itemize}
This paper is organized as follows. The related work is reviewed in Section \ref{Section2}. In Section \ref{Section3}, we introduce the notations and present our main results.
We provide the proofs in Section \ref{Section4}. Section \ref{Section5} concludes this paper. 
Some discussions and proofs are deferred to the Appendix, including a systematic comparison with the related work.



 
\section{Related Work}\label{Section2}
This section introduces the related work on generalization performance analysis of pairwise learning based on different approaches. 

Algorithmic stability is a popular approach to study the generalization performance of pairwise learning. It is also a fundamental concept in statistical learning theory \cite{bousquet2002stability,bousquet2020sharper,klochkov2021stability}, which has a deep
connection with learnability \cite{rakhlin2005stability,shalev-shwartz2010learnability,shalev-shwartz2015understanding}. A training algorithm is
stable if small changes in the training set result in small differences in the output predictions
of the trained model \cite{bousquet2002stability}. \cite{agarwal2009generalization,gao2013uniform} establish the relationship between generalization and stability for ranking. \cite{jin2009regularized,wang2019multitask} study the regularized metric learning based on stability. \cite{yang2021stability,huai2020pairwise} consider differential privacy problems in pairwise setting. \cite{shen2020stability} uses stability to study the trade-off between the generalization error and optimization error for a variant of pairwise SGD. \cite{lei2020sharper} starts the studying of pairwise learning framework via algorithmic stability. They provide an improved stability analysis based on \cite{bousquet2020sharper}, and further use it to establish learning rates  for RRM and SGD. \cite{lei2021generalization} further provides generalization guarantees for pairwise SGD under milder assumptions. Although algorithmic stability has been widely employed in pairwise learning, it generally requires convexity assumptions \cite{lei2021generalization}, which means that the above studies  are mostly limited to convex learning. Moreover, the strong convexity condition is often required when establishing faster learning rates. However, it is known that the strong convexity condition is too restrictive \cite{karimi2016linear}.

Another popular approach employed for pairwise learning is uniform convergence \cite{bartlett2002rademacher,bartlett2005local,liu2020fast,li2021towards,liu2021effective}. An advantage of uniform convergence is that it can imply meaningful learning rates for nonconvex learning \cite{mei2018the,foster2018uniform,davis2021graphical,lei2020sharper,lei2021generalization}. In the related work of uniform convergence, \cite{cao2016generalization,clemencon2008ranking,clemencon2005ranking,rejchel2012on,ye2019fast,verma2015sample,zhou2016generalization,yang2021learning,li2021sharper,liu2017generalization,ying2016stochastic,liu2018fast1,lei2021stochastic} focus on the specific instantiations of pairwise learning, i.e., metric learning, ranking or AUC maximization. They often bound the generalization gap by its supremum over the whole (or a subset) of the hypothesis space. Then, some space complexity measures, including VC dimension, covering number and Rademacher complexity, can be adopted to prove the learning rates of the generalization performance. \cite{lei2018generalization} studies the pairwise learning framework via the uniform convergence technique. But they require a strong convexity assumption. In a very recent work, \cite{lei2021generalization} develops uniform convergence of gradients for pairwise learning based on \cite{lei2021learning}, and further use it to investigate the learning rates of SGD in nonconvex pairwise learning. The uniform convergence of gradients has recently drawn increasing attention in nonconvex learning \cite{mei2018the,foster2018uniform,lei2021learning,xu2020toward,davis2021graphical} and stochastic optimization \cite{zhang2017empirical,zhang2019stochastic,liu2018fast}, which is a gap between the gradients of the population risk and the gradients of the empirical risk. However, these works are limited to the pointwise learning setting. In this paper, we study the more complex pairwise learning and provide improved uniform convergence of gradients than \cite{lei2021generalization}, based on which we investigate the learning rates for generalization performance of nonconvex pairwise learning. As discussed before, the dependency in the empirical risk hinders the standard i.i.d technique. To overcome this difficulty, we need to decouple this dependency so that the standard generalization analysis established for independent data can be applied to this context. Furthermore, we develop different uniform convergence of gradients under different assumptions. For the demand of the proof, we also create two more general forms of the Bernstein inequality of pairwise learning, which may be of independent interest and benefit the Bernstein inequality's broader applicability (please refer to Appendix \ref{appendix1} for details). 

Except for the algorithmic stability and uniform convergence, convex analysis is employed in online pairwise learning \cite{kar2013on,wang2012generalization}. The tool of integral operator is also used to study the generalization of pairwise learning, but is often limited to the specific least square loss functions \cite{ying2016online,guo2020distributed}. 

\section{Main Results}\label{Section3}
\subsection{Preliminaries}
Let $P$ be a probability measure defined over a sample space $\mathcal{Z}$ and $P_n$ be the corresponding empirical probability measure.
 Let $f(\cdot, z, z') : \mathcal{W} \mapsto R$ be a random objective function depending on random variables
$z, z' \in \mathcal{Z}$, where $\mathcal{W}$ is a parameter space of dimension $d$.
In pairwise learning, we aim to minimize the following expected risk 
\begin{align}\label{eqs1}
F(\mathbf{w}) = \mathbb{E}_{z,z'}[f(\mathbf{w};z,z')],
\end{align}
where $\mathbb{E}_{z,z'}$ denotes the expectation with respect to (w.r.t.) $z,z' \sim P$.
In (\ref{eqs1}),
$F(\mathbf{w})$ is also referred to as population risk. $z$ and $z'$ can be considered as samples, $\mathbf{w}$ can be interpreted as a model or hypothesis, and $f(\cdot, \cdot , \cdot)$ can be viewed
as a loss function. 

A well-known example of (\ref{eqs1}) is the pairwise supervised learning. 
Specifically, in the supervised learning, $\mathcal{Z} = \mathcal{X} \times \mathcal{Y}$ with $\mathcal{X} \subset \mathbb{R}^{d'}$ being the input space and $\mathcal{Y} \subset \mathbb{R}$ being the output space ($d'$ may not equal to $d$). Let $S=\{ z_1,...,z_n \}$ be a training dataset drawn independently according to $P$, based on which we wish to build a prediction function $h: \mathcal{X} \mapsto \mathbb{R}$ or $h: \mathcal{X} \times \mathcal{X} \mapsto \mathbb{R}$. Considering the parametric models, in which the predictor $h_{\mathbf{w}}$ can be indexed by a parameter $\mathbf{w} \in \mathcal{W}$, and defining $\ell(\mathbf{w};z,z')$ as the loss that measures the quality of $h_{\mathbf{w}}$ over $z,z' \in \mathcal{Z}$, where $\ell:  \mathcal{W} \times \mathcal{Z} \times \mathcal{Z} \mapsto \mathbb{R}$, the corresponding expected risk of supervised learning can be written as
\begin{align}\label{eqs2}
F(\mathbf{w}) = \mathbb{E}_{z,z'}[\ell(\mathbf{w};z,z')].
\end{align}
 In contrast to the traditional pointwise learning problems where the quality of a model parameter $\mathbf{w}$ is measured over an individual point,
  a distinctive property of (\ref{eqs2}) is that the performance of $h_{\mathbf{w}}$ should be quantified on pairs of data samples. 
Note that the
minimization of (\ref{eqs1}) is more general than supervised learning in (\ref{eqs2}) and could be
more challenging to handle \cite{shalev-shwartz2010learnability,shalev-shwartz2015understanding}.

From (\ref{eqs1}), we know that the population risk $F(\mathbf{w})$ measures the prediction performance of $\mathbf{w}$ over the underlying distribution. However,
$P$ is typically not available and what we get is only a set of i.i.d. training samples $S$. In practice, we minimize the following empirical risk as an approximation \cite{wainwright2019high}
\begin{align}\label{ineq345}
F_S(\mathbf{w}) = \frac{1}{n(n-1)}\sum_{i,j\in [n], i \neq j} f(\mathbf{w};z_i,z_j),
\end{align}
where $[n] = \{1,...,n  \}$. In optimizing (\ref{ineq345}), some popular algorithms are proposed including empirical risk minimizer (ERM), gradient descent (GD), and stochastic gradient descent (SGD). For this reason, we will provide generalization analysis for these algorithms.
We now introduce some notations used in this paper.
Denote $\| \cdot \|$ to be the $L_2$ norm in $\mathbb{R}^d$, i.e., $\| \mathbf{w} \| = (\sum_{i=1}^d | w_i|^2)^{\frac{1}{2}}$. Let $B(\mathbf{w}_0, R) := \{ \mathbf{w} \in \mathbb{R}^d
: \|\mathbf{w} - \mathbf{w}_0\| \leq R \}$ denote a ball with center $\mathbf{w}_0 \in \mathbb{R}^d$ and radius
$R$. We assume that there is a radius $R_1$ such that $ \mathcal{W} \subseteq B(\mathbf{w}^{\ast}, R_1)$. Let $e$ be the base of the natural logarithm.

For a better understanding of the pairwise learning framework, we now provide two examples to explain it.
\begin{itemize}
\item \textbf{Bipartite ranking.} In ranking problems, we aim to learn a good estimator $h_{\mathbf{w}}:\mathcal{X}\times \mathcal{X} \mapsto \mathbb{R}$ which can correctly
predict the ordering of pairs of binary labeled samples, i.e., predicting $y> y'$ if $h_{\mathbf{w}}(x,x') > 0$. The performance of $h_{\mathbf{w}}$ at examples $(z,z')$ can be measured by choosing the $0-1$ loss. However, the $0-1$ loss is hard to be optimized in practice, one often employs surrogate losses \cite{cortes2016structured}. By considering the convex surrogate losses $\ell: \mathbb{R} \mapsto \mathbb{R}_{+}$, the loss function of ranking is of the form $f(\mathbf{w};z,z') = \ell(sign (y - y')h_{\mathbf{w}}(x,x'))$, where $sign(x)$ is the sign of $x$. Common choices of the surrogate loss $\ell$ include the hinge loss and the logistic loss \cite{mohri2012foundations}.

\item \textbf{Metric learning.}
Let's consider the supervised metric learning with the label space $\mathcal{Y}=\{ -1,+1\}$. Under this setting, we want to learn a distance metric function $h_{\mathbf{w}}(x,x') = \langle \mathbf{w}, (x-x')(x-x')^T \rangle$ such that  a pair $(x, x')$ of inputs from the same class $(y = y')$ are close to
each other while a pair from different classes $(y \neq y')$ have a large distance $h_{\mathbf{w}}(x, x')$ \cite{lei2021generalization}, where $x^T$ denotes the transpose of $x \in \mathbb{R}^d$ and $\mathbf{w} \in \mathbb{R}^{d\times d}$. Similarly, considering the convex surrogate loss function $\ell$, a common choice of the loss function in supervised metric learning is of the form $f(\mathbf{w};z,z') = \ell(yy'(1-h_{\mathbf{w}}(x,x')))$ \cite{jin2009regularized,lei2021generalization}. Moreover, 
one can refer to \cite{li2021sharper} for examples of unsupervised metric learning, where the authors study the similarity-based clustering learning under the framework of pairwise learning.
\end{itemize}

\subsection{Uniform Convergence of Gradients}
Uniform convergence of gradients measures the deviation between the population gradients $\nabla F$ and the empirical gradients $\nabla F_S$, where $\nabla$ denotes the  gradient operator. In this subsection, we aim to provide improved uniform convergence of gradients than the associated one in \cite{lei2021generalization}.
Before providing the main theorems, we first introduce a crucial assumption.
\begin{assumption}\label{assum1}
For all $\mathbf{w}_1, \mathbf{w}_2 \in \mathcal{W}$, we assume that $\frac{\nabla f(\mathbf{w}_1;z,z') - \nabla f(\mathbf{w}_2;z,z')}{\| \mathbf{w}_1 - \mathbf{w}_2\|}$ is a $\gamma$-sub-exponential random vector. That is there exists $\gamma > 0$ such that for any unit vector $\mathbf{u} \in B(0,1)$ and $\mathbf{w}_1, \mathbf{w}_2 \in \mathcal{W}$,
\begin{align*}
\mathbb{E}\Big \{ \exp \Big(\frac{|\mathbf{u}^T (\nabla f (\mathbf{w}_1;z,z') - \nabla f (\mathbf{w}_2;z,z'))|}{\gamma \|\mathbf{w}_1 - \mathbf{w}_2 \| } \Big) \Big \} \leq 2.
\end{align*}
\end{assumption}
\begin{remark}\rm{}\label{remark1}
This assumption is stronger than the smoothness of the population risk, but much milder than the uniform smoothness condition (Assumption \ref{assum4}). Please refer to Section \ref{section4.1} for the proof.
\end{remark}
Based on Assumption \ref{assum1}, we have the first theorem on uniform convergence of gradients.
\begin{theorem}\label{theorem1}
Suppose Assumption \ref{assum1} holds.
Then for any $ \delta \in (0,1)$, with probability $1-\delta$, for all $\mathbf{w} \in \mathcal{W}$, we have
\begin{multline*}
\hphantom{{}={}}\left \| (\nabla F (\mathbf{w} )-\nabla F_S(\mathbf{w}))  -  (\nabla F (\mathbf{w}^{\ast})-\nabla F_S(\mathbf{w}^{\ast})) \right \| \\
\leq c\gamma \max \Big\{\|\mathbf{w} - \mathbf{w}^{\ast}\|  ,\frac{1}{n}\Big\} 
\Big( \sqrt{\frac{d+\log \frac{4 \log_2(\sqrt{2}R_1 n + 1)}{\delta}}{n}} \\+ \frac{d+\log \frac{4 \log_2(\sqrt{2}R_1 n + 1)}{\delta}}{n} \Big ),
\end{multline*}
where $c$ is an absolute constant and $\mathbf{w}^{\ast} \in \arg \min_{\mathcal{W}} F(\mathbf{w})$.
\end{theorem}
\begin{remark}\rm{}
Uniform convergence of gradients is firstly studied in convex learning \cite{zhang2017empirical,zhang2019stochastic}. Recently, uniform convergence of gradients of nonconvex learning is also proposed based on different techniques. Specifically, \cite{mei2018the} is based on covering numbers, \cite{foster2018uniform} is based on a chain rule for vector-valued Rademacher complexity, \cite{lei2021learning} is based on Rademacher chaos complexity, \cite{davis2021graphical} is based on the gradient of the Moreau envelops, and \cite{xu2020toward} is based on a novel uniform localized convergence technique. However, the above-mentioned works are limited to the pointwise learning case. In Theorem \ref{theorem1}, we present the uniform convergence of gradients for the more complex pairwise learning. 
As discussed in Section \ref{Section2}, a key difference between pointwise learning and pairwise learning is that the gradient of the empirical risk in pairwise learning (see (\ref{ineq345})) involves $\mathcal{O}(n^2)$ dependent terms, which makes the proof of Theorem \ref{theorem1} more challenging. 
\end{remark}

We now introduce a Bernstein condition at the optimal point, based on which we will show Theorem \ref{theorem2}.
\begin{assumption}\label{assum2}
The gradient at $\mathbf{w}^{\ast}$ satisfies the Bernstein condition, i.e., there exists $D_{\ast} >0$ such that for all $2 \leq k \leq n$,
\begin{align*}
\mathbb{E}[\| \nabla f(\mathbf{w}^{\ast};z,z') \|^k] \leq \frac{k!}{2}\mathbb{E}[\| \nabla f(\mathbf{w}^{\ast};z,z') \|^2] D_{\ast}^{k-2}.
\end{align*}
\end{assumption}
\begin{remark}\rm{}
Assumption \ref{assum2} is pretty mild since $D_{\ast} >0$ only depends on gradients at $\mathbf{w}^{\ast}$. Moreover,
the Bernstein condition is milder than the bounded assumption of random variables and is also satisfied by various unbounded variables \cite{wainwright2019high}.
Please refer to \cite{wainwright2019high} for more discussions on this assumption.
\end{remark}
\begin{theorem}\label{theorem2}
Suppose Assumptions \ref{assum1} and \ref{assum2} hold. For any $\delta>0$, with probability at least $1-\delta$, for all $\mathbf{w} \in \mathcal{W}$, we have
\begin{multline*}
\|\nabla F (\mathbf{w} )-\nabla F_S(\mathbf{w})  \| 
\leq  c\gamma \max \Big\{\|\mathbf{w} - \mathbf{w}^{\ast}\|  ,\frac{1}{n}\Big\} \\
\times \Big( \sqrt{\frac{d+\log \frac{8 \log_2(\sqrt{2}R_1 n + 1)}{\delta}}{n}} + \frac{d+\log \frac{8 \log_2(\sqrt{2}R_1 n + 1)}{\delta}}{n}  \Big) \\
+\frac{4D_{\ast} \log \frac{4}{\delta}}{n} + \sqrt{\frac{8\mathbb{E}[\| \nabla f(\mathbf{w}^{\ast};z,z')  \|^2] \log \frac{4}{\delta}}{n}},
\end{multline*}
where $c$ is an absolute constant and $\mathbf{w}^{\ast} \in \arg \min_{\mathbf{w} \in \mathcal{W}} F(\mathbf{w})$.
\end{theorem}
\begin{remark}\rm{}
There is only one uniform convergence of gradients for pairwise learning, developed in \cite{lei2021generalization}. 
We now compare our uniform convergence of gradients with \cite{lei2021generalization}. Under uniformly smooth assumption (Assumption \ref{assum4}), \cite{lei2021generalization} shows that with probability at least $1-\delta$
\begin{align}\label{ineq1}
&\sup_{\mathbf{w}\in B(0,R)} \|  \nabla F (\mathbf{w} )-\nabla F_S(\mathbf{w}) \| \\\nonumber
\leq &\frac{c(\beta R + b)}{\sqrt{n}} \Big(  2+\sqrt{96e(\log 2 +d\log(3e))} + \sqrt{\log(1/\delta)}\Big),
\end{align}
where $b = \sup_{z,z' \in \mathcal{Z}}\| \nabla f(0;z,z') \|$. Compared with (\ref{ineq1}), we successfully relax the uniform smoothness assumption to a milder Assumptions \ref{assum1}. 
Moreover, the factor in (\ref{ineq1}) is $c(\beta R+b)$, while in Theorem \ref{theorem2} is $c\gamma \max \{\|\mathbf{w} - \mathbf{w}^{\ast}\|  ,\frac{1}{n}\}$, not involving a term $\sup_{z,z' \in \mathcal{Z}}\| \nabla f(0;z,z') \|$ that may be very large. And we emphasize that it is the construction of the factor that allows us to derive improved learning rates when considering Assumption \ref{assum3}. 
\end{remark}

In the following, we further provide an improved uniform convergence of gradients when the PL curvature condition (gradient dominance condition) is satisfied.
\begin{assumption}\label{assum3}
Fix a set $\mathcal{W}$. For any function $f: \mathcal{W} \mapsto \mathbb{R}$, let $f^{\ast} = \min_{\mathbf{w} \in \mathcal{W}} f(\mathbf{w})$. $f$ satisfies the Polyak-{\L}ojasiewicz (PL) condition with parameter $\mu > 0$ on $\mathcal{W}$ if 
\begin{align*}
f (\mathbf{w}) - f^{\ast} \leq \frac{1}{2 \mu} \| \nabla f(\mathbf{w}) \|^2, \quad  \forall \mathbf{w} \in \mathcal{W}.
\end{align*}
\end{assumption}
\begin{remark}\rm{}
PL condition is also referred to as ``gradient dominance condition'' \cite{foster2018uniform}. This condition means that the suboptimality of function values
can be bounded by the squared magnitude of gradients, which
can be used to bound how far away the nearest minimizer is in terms of the optimality gap. It is one of the weakest curvature conditions and is widely employed in nonconvex learning \cite{xu2020toward,lei2021sharper,lei2021generalization,reddi2016stochastic,zhou2018generalization,karimi2016linear,charles2018stability,lei2021learning}, to mention but a few. Under suitable assumptions on the input, many popular nonconvex objective functions satisfy PL condition, including neural networks with one hidden 
layer \cite{li2017convergence}, ResNets with linear activations \cite{hardt2016identity}, robust regression \cite{liu2016quadratic}, linear dynamical systems \cite{hardt2016gradient}, matrix factorization \cite{liu2016quadratic},  phase
retrieval \cite{sun2018geometric}, blind deconvolution \cite{li2019rapid}, mixture of two Gaussians \cite{balakrishnan2017statistical}, etc.
\end{remark}
\begin{theorem}\label{theorem3}
Assume Assumptions \ref{assum1} and \ref{assum2} hold.
Suppose the population risk $F$ satisfies Assumption \ref{assum3} with parameter $\mu$. Then for any $\delta >0 $, when $n \geq \frac{c\gamma^2(d+ \log \frac{8 \log_2(\sqrt{2}R_1 n + 1)}{\delta})}{\mu^2}$, with probability at least $1- \delta$, for all $\mathbf{w} \in \mathcal{W}$, we have
\begin{align}\label{inqe11}\nonumber
&\hphantom{{}={}}\left\| \nabla F (\mathbf{w} )- \nabla F_S(\mathbf{w}) \right\|\leq \left\| \nabla F_S(\mathbf{w}) \right\| +  \frac{\mu}{n}\\
&  + \frac{8D_{\ast}\log(4/\delta)}{n} + 4\sqrt{\frac{2 \mathbb{E} [ \| \nabla f(\mathbf{w}^{\ast};z,z') \|^2 ] \log(4/\delta)}{n}}.
\end{align}
\end{theorem}
\begin{remark}\rm{}\label{ekurh948}
It is clear that (\ref{inqe11}) implies 
\begin{align}\label{iqe23}\nonumber
&\hphantom{{}={}}\|  \nabla F(\mathbf{w})  \| \leq 2\left\| \nabla F_S(\mathbf{w}) \right\| +  \frac{\mu}{n}\\
& + \frac{8D_{\ast}\log(4/\delta)}{n} + 4\sqrt{\frac{2 \mathbb{E} [ \| \nabla f(\mathbf{w}^{\ast};z,z') \|^2 ] \log(4/\delta)}{n}}.
\end{align}
Typically, we call $\| \nabla F_S(\mathbf{w}) \|^2$ the optimization error and $\| \nabla F_S(\mathbf{w})-\nabla F(\mathbf{w}) \|^2$ the statistical error (or generalization error) \cite{lei2021learning}, since the former is related to the optimization algorithm to optimize $F_S$, and the latter is related to approximating the true gradient with its empirical form.
In Theorem \ref{theorem3}, $ \left\| \nabla F_S(\mathbf{w}) \right\|$ can be tiny since 
the optimization algorithms, such as GD and SGD, can optimize it to be small enough. $\mathbb{E} [ \| \nabla f(\mathbf{w}^{\ast};z,z') \|^2 ]$ is also tiny since it depends on the gradient on the optima $\mathbf{w}^{\ast}$ and involves an expectation operator \cite{zhang2017empirical,lei2021generalization,liu2018fast,lei2020fine}.  Therefore, compared with Theorems \ref{theorem1} and \ref{theorem2}, and (\ref{ineq1}), this uniform convergence of gradients is clearly tighter. Moreover, the fact that our established convergence of gradients scales tightly with the optimal parameter, i.e., the gradient norms at the optima $\mathbf{w}^{\ast}$, largely contributes to derive faster $\mathcal{O}(1/n^2)$ rates of this paper, which is a remarkable advance compared to (\ref{ineq1}). The appearance of $\mathbb{E} [ \| \nabla f(\mathbf{w}^{\ast};z,z') \|^2 ]$ requires technical analysis. Additionally, an obvious shortcoming of uniform convergence is that it often implies learning rates with a square-root dependency on the dimension $d$ when considering general problems \cite{feldman2016generalization}, as shown in (\ref{ineq1}), and Theorems \ref{theorem1} and \ref{theorem2}. Another distinctive improvement of Theorem \ref{theorem3} is that we successfully remove the dimension $d$ when the population risk $F$ satisfies the PL condition and the sample size $n$ is large enough. Based on Theorem \ref{theorem3}, we will provide dimension-independent learning rates for ERM, GD, and SGD. In addition to these algorithms, the uniform convergence of gradients in this paper can be employed to study other optimization algorithms, such as variance reduction variants and  momentum-based optimization algorithms \cite{nemirovski2008robust}, which would also be very interesting. 
\end{remark}
\subsection{Empirical Risk Minimizer}\label{section3.3}
Generalization performance means the generalization behavior of the trained model on testing examples.
Let $\mathbf{w}(S)$ be the learned model produced by some algorithms on the training set $S$. In Section \ref{section3.3}-\ref{section3.5}, we first consider the general nonconvex learning problems and present the learning rate for the gradient norm of the population risk, i.e., $\|  \nabla F(\mathbf{w}(S)) \|$. After that, we study the noconvex learning with the PL condition and provide learning rates for the generalization performance gap
$F(\mathbf{w}(S)) - F(\mathbf{w}^{\ast})$, where $\mathbf{w}^{\ast} = \arg \min_{\mathbf{w} \in \mathcal{W}} F(\mathbf{w})$.
In this section,
we consider the ERM problem.
In ERM, we focus on the optima $\hat{\mathbf{w}}^{\ast}$ of the empirical risk $F_S$, i.e., $\hat{\mathbf{w}}^{\ast} \in \arg \min_{\mathbf{w} \in \mathcal{W}} F_S(\mathbf{w})$. 
\begin{theorem}\label{theorem4}
Suppose the empirical risk minimizers $\hat{\mathbf{w}}^{\ast}$ exists. Assume Assumptions \ref{assum1} and \ref{assum2} hold. For any $\delta \in (0,1)$, with probability at least $1-\delta$, we have
\begin{align*}
\|  \nabla F(\hat{\mathbf{w}}^{\ast}) \| = \mathcal{O} \Big( \sqrt{\frac{d+\log \frac{\log n}{\delta}}{n}}   \Big).
\end{align*}
\end{theorem}
\begin{remark}\rm{}
When Assumption \ref{assum1} and \ref{assum2} hold, Theorem \ref{theorem4} shows that the learning rate of $\|  \nabla F(\hat{\mathbf{w}}^{\ast}) \|$ is of order $\mathcal{O} ( \sqrt{\frac{d+\log \frac{1}{\delta}}{n}}   )$ ($\log n$ is small and can be ignored typically). Note that this bound does not require the uniform smoothness condition (Assumption \ref{assum4}). Although it is hard to find $\hat{\mathbf{w}}^{\ast}$ in nonconvex learning, this learning rate is meaningful by assuming the ERM has been found. Moreover, this learning rate may be comparable to the classical one $\mathcal{O}(\sqrt{\frac{d \log n \log(d/\delta) }{n}})$ in the stochastic convex optimization \cite{shalev-shwartz2009stochastic}, without requiring the convexity condition.
\end{remark}
\begin{theorem}\label{theorem5}
Suppose Assumptions \ref{assum1} and \ref{assum2} hold, and the population risk $F(\mathbf{w})$ statisfies Assumption \ref{assum3} with parameter $\mu$. For any $\delta \in (0,1)$, with probability at least $1-\delta$, when $n \geq \frac{c\gamma^2(d+ \log(\frac{8 \log(\sqrt{2}n R_1 +1)}{\delta}))}{\mu^2}$, we have
\begin{align*}
F(\hat{\mathbf{w}}^{\ast}) -F(\mathbf{w}^{\ast})  = \mathcal{O} \Big( \frac{\log^2\frac{1}{\delta}}{ n^2}   + \frac{ \mathbb{E} [ \| \nabla f(\mathbf{w}^{\ast};z,z') \|^2] \log\frac{1}{\delta}}{ n} \Big).
\end{align*}
If further assume $\mathbb{E} [ \| \nabla f(\mathbf{w}^{\ast};z,z') \|^2 ]= \mathcal{O} \left(\frac{1}{n} \right)$, we have 
\begin{align*}
F(\hat{\mathbf{w}}^{\ast}) -F(\mathbf{w}^{\ast})  = \mathcal{O} \Big( \frac{\log^2(1/\delta)}{ n^2}  \Big).
\end{align*}
\end{theorem}
\begin{remark}\rm{}
Theorem \ref{theorem5} shows that when population risk $F(\mathbf{w})$ satisfies the PL condition, we can provide much faster learning rate than Theorem \ref{theorem4}. The learning rate can even up to $\mathcal{O}(\frac{1}{n^2})$. Note that the term $\mathbb{E} [ \| \nabla f(\mathbf{w}^{\ast};z,z') \|^2$ is tiny since it depends on the optima $\mathbf{w}^{\ast}$ and involves the expectation operator \cite{zhang2017empirical,lei2021generalization,liu2018fast,lei2020fine}, as discussed in Remark \ref{ekurh948}. Moreover, from (\ref{nonioggehe}), one can see that if $f$ is nonnegative and $\beta$-smooth, $\mathbb{E} [ \| \nabla f(\mathbf{w}^{\ast},z) \|^2 ] \leq  4 \beta  F(\mathbf{w}^{\ast})$. It is notable that the assumption $F(\mathbf{w}^{\ast}) \leq \mathcal{O} \left(\frac{1}{n} \right)$ is also common and can be found in \cite{zhang2017empirical,zhang2019stochastic,lei2020sharper,srebro2010optimistic,lei2021generalization,liu2018fast,lei2020fine}, which is natural since $F(\mathbf{w}^{\ast})$ is the minimal population risk. We now compare our result with the most related work \cite{lei2021sharper,lei2018generalization}. \cite{lei2018generalization} studies the learning rate of generalization performance gap of regularized empirical risk minimizers (RRM) via uniform convergence technique. Under the Lipschitz continuity condition and the strong convexity condition, Theorem 1 and Theorem 2 in \cite{lei2018generalization} provide $\mathcal{O}(\frac{\log(1/\delta)}{n})$ order rates. \cite{lei2021sharper} studies the generalization performance gap of RRM via algorithmic stability. Under the Lipschitz continuity and strong convexity conditions, Theorem 3 in \cite{lei2021sharper} provides $\mathcal{O}(\frac{\log n \log(1/\delta)}{\sqrt{n}})$ order rates. By the comparison, we have established much faster learning rates, significantly, under a nonconvex learning setting.  
\end{remark}
\subsection{Gradient Descent}
\begin{algorithm}[tb]
   \caption{GD for Pairwise Learning}
   \label{alg:example0}
   \textbf{Input:} initial point $\mathbf{w}_1 = 0$, step sizes $\{ \eta_t \}_t$, and dataset $S = \{z_1,...,z_n \}$
\begin{algorithmic}[1]
    \FOR{$t=1,...,T$}
   \STATE update $\mathbf{w}_{t+1} = \mathbf{w}_t - \eta_t \nabla F_S(\mathbf{w}_t)$
   \ENDFOR
\end{algorithmic}
\end{algorithm}
We now analyze the generalization performance of gradient descent of pairwise learning, where the algorithm is shown in Algorithm \ref{alg:example0}. Denote $A\asymp B$ if there exists universal constants $C_1, C_2 >0$ such that $C_1A \leq B \leq C_2 A$.
Similarly, we first introduce a necessary assumption. 
\begin{assumption}[Smoothness]\label{assum4}
Let $\beta > 0$.
For any sample $z,z' \in \mathcal{Z}$ and $\mathbf{w}_1, \mathbf{w}_2 \in \mathcal{W}$, there holds that
\begin{align*}
\|\nabla f(\mathbf{w}_1;z,z') - \nabla f(\mathbf{w}_2;z,z') \| \leq \beta \| \mathbf{w}_1 - \mathbf{w}_2 \|.
\end{align*}
\end{assumption}
\begin{remark}\rm{}
The uniform smoothness condition is commonly used in nonconvex learning \cite{mei2018the,foster2018uniform,lei2021generalization,lei2021learning,davis2021graphical,hardt2016train}.
As discussed in Section \ref{section4.1}, Assumption \ref{assum4} implies Assumption \ref{assum1}. Thus, the established uniform convergences of gradients is also correct under Assumption \ref{assum4}. In the following, we require this assumption to derive the optimization error bound, i.e., $\|  \nabla F_S(\mathbf{w}(S)) \|$.
\end{remark}
\begin{theorem}\label{theorem6}
Suppose Assumptions \ref{assum2} and \ref{assum4} hold and
the objective function $f$ is nonnegative.
Let $\{ \mathbf{w}_t \}_t$ be the sequence produced by Algorithm \ref{alg:example0} with $\eta_t = \eta_1 t^{-\theta}$, $\theta \in (0,1)$ and $\eta_1 \leq 1/\beta$. For any $\delta \in (0,1)$, with probability at least $1-\delta$,
when $T \asymp (nd^{-1})^{\frac{1}{2(1-\theta)}}$, we have
\begin{align*}
&\frac{1}{\sum_{t = 1}^T \eta_t}\sum_{t = 1}^T \eta_t\| \nabla F(\mathbf{w}_t) \|^2 
\leq \mathcal{O} \Big(\frac{d + \log \frac{ \log n }{\delta}}{\sqrt{nd}}\Big).
\end{align*}
\end{theorem}
\begin{remark}\rm{}
To our best knowledge, this is the first work that investigates the learning rates of GD for nonconvex pairwise learning. Theorem \ref{theorem6} shows that for pairwise GD, one should select an appropriate iterative number for early-stopping to achieve a good learning rate. In the proof, (\ref{jizbnu}) reveals that we should balance the optimization error (optimization) and the statistical error (generalization), which demonstrates the reason for early-stopping. According to Theorem \ref{theorem6}, the optimal iterative number should be chosen as $T \asymp (nd^{-1})^{\frac{1}{2(1-\theta)}}$ for polynomially decaying step sizes. 
\end{remark}

\begin{theorem}\label{theorem7}
Suppose Assumptions \ref{assum2} and \ref{assum4} hold and
the objective function $f$ is nonnegative.
Assume the empirical risk $F_S$ satisfies Assumption \ref{assum3} with parameter $\mu$. Let $\{ \mathbf{w}_t \}_t$ be the sequence produced by Algorithm \ref{alg:example0} with $\eta_t = 1/\beta$. For any $\delta \in (0,1)$, with probability at least $1-\delta$, we have
\begin{multline*}
F(\mathbf{w}_{T+1}) - F(\mathbf{w}^{\ast}) \leq \mathcal{O} \Big((1-\frac{\mu}{\beta})^{T} \Big)  \\+  \mathcal{O} \Big(\frac{\log^2(1/\delta)}{n^2} + \frac{ F(\mathbf{w}^{\ast}) \log(1/\delta)}{n}\Big) .
\end{multline*}
If further assume $F(\mathbf{w}^{\ast}) = \mathcal{O} \Big(\frac{1}{n} \Big)$ and choose $T \asymp \log n$, we have
\begin{align*}
F(\mathbf{w}_{T+1}) - F(\mathbf{w}^{\ast})   = \mathcal{O} \Big( \frac{\log^2(1/\delta)}{n^2} \Big).
\end{align*}
\end{theorem}
\begin{remark}\rm{}\label{reemklwjfw9}
For brevity, we show Theorem \ref{theorem7} with a step size $\eta_t = 1/\beta$. Indeed, Theorem \ref{theorem7} is correct for any $0< \eta_t \leq 1/\beta$.
Theorem \ref{theorem7} reveals that when $F_S$ satisfies the PL condition, the
generalization performance gap of GD is of a faster rate $\mathcal{O}(\frac{ F(\mathbf{w}^{\ast}) \log(1/\delta)}{n})$ than Theorem \ref{theorem6}. If we suppose the optimal population risk is small as assumed in \cite{zhang2017empirical,zhang2019stochastic,lei2020sharper,srebro2010optimistic,lei2021generalization,liu2018fast,lei2020fine}, we further obtain faster learning rate of order $\mathcal{O}(\frac{\log^2(1/\delta)}{n^2})$. 
\end{remark}



\subsection{Stochastic Gradient Descent}\label{section3.5}
\begin{algorithm}[tb]
   \caption{SGD for Pairwise Learning}
   \label{alg:example}
   \textbf{Input:} initial point $\mathbf{w}_1 = 0$, step sizes $\{ \eta_t \}_t$, and dataset $S = \{z_1,...,z_n \}$
\begin{algorithmic}[1]
    \FOR{$t=1,...,T$}
    \STATE  draw  $(i_t,j_t)$ from the uniform distribution over the set $\{ (i,j): i,j\in [n] , i\neq j\}$ \\
   \STATE update $\mathbf{w}_{t+1} = \mathbf{w}_t - \eta_t \nabla f(\mathbf{w}_t;z_{i_t},z_{j_t})$
   \ENDFOR
\end{algorithmic}
\end{algorithm}
Stochastic gradient descent optimization algorithm has found wide application in machine learning due to its simplicity in implementation, low memory requirement and low computational complexity per iteration, as well as good practical behavior \cite{zhang2004solving,bach2013non,bottou2018optimization,harvey2019tight}. The description of SGD of pairwise learning is shown in Algorithm \ref{alg:example}.
We also first introduce a necessary assumption.
\begin{assumption}\label{assum5}
Assume the existence of $G>0$ and $\sigma > 0$ satisfying
\begin{align}\label{assum145}
\sqrt{\eta_t} \|\nabla f(\mathbf{w}_{t};z,z') \| \leq G,  \forall t \in \mathbb{N}, z,z'\in \mathcal{Z},\\\label{assum1456}
\mathbb{E}_{i_t,j_t} \left[ \| \nabla f(\mathbf{w}_{t};z_{i_t},z_{j_t}) - \nabla F_S (\mathbf{w}_{t}) \|^2 \right]\leq \sigma^2, \forall t \in \mathbb{N},
\end{align}
where $\mathbb{E}_{i_t,j_t}$ denotes the expectation w.r.t. $i_t$ and $j_t$.
\end{assumption}
\begin{remark}\rm{}\label{r08gur8ghiothjb}
In Assumption \ref{assum5}, (\ref{assum145}) is much milder than the bounded gradient assumption (see Appendix \ref{appendix2}) since $\eta_t$ is typically small \cite{lei2021generalization}, such as the setting of this paper. (\ref{assum1456}) is a common assumption in the generalization performance  analysis of SGD \cite{zhou2018generalization,lei2021generalization,li2021improved}.
\end{remark}
\begin{theorem} \label{theorem8}
Suppose Assumptions \ref{assum2}, \ref{assum4}  and \ref{assum5} hold. Let $\{ \mathbf{w}_t\}_t$ be the sequence produced by Algorithm \ref{alg:example} with $\eta_t = \eta_1 t^{- \theta} , \theta \in (0,1)$ and $\eta_1 \leq \frac{1}{2\beta}$. Then, for any $\delta >0$, with probability $1 - \delta$, when $T \asymp  (nd^{-1})^{\frac{1}{2-2\theta}}$, we have
\begin{align*}
&\Big(\sum_{t = 1}^T \eta_t \Big)^{-1} \sum_{t = 1}^T \eta_t \| \nabla F(\mathbf{w}_t) \|^2  \\
=
&\begin{cases}
             \mathcal{O} \Big(\Big(\sqrt{\frac{d}{n}}\Big)^{\frac{\theta}{1-\theta }} \log^3(1/\delta) \Big), &\quad \text{if   } \theta < 1/2,   \\ 
             \mathcal{O} \Big( \sqrt{\frac{d}{n}} \log(T/\delta)  \log^3(1/\delta) \Big), &\quad \text{if   }\theta = 1/2,\\ 
             \mathcal{O} \Big(\sqrt{\frac{d}{n}} \log^3(1/\delta) \Big),  &\quad \text{if   }\theta > 1/2. 
\end{cases}
\end{align*}
\end{theorem}
\begin{remark}\rm{}
Similar to Theorem \ref{theorem6}, Theorem \ref{theorem8} also implies a trade-off between the optimization error (optimization) and the statistical error (generalization) for SGD, as revealed in (\ref{ru4yt9785gui}-\ref{ineq987}). Theorem \ref{theorem8} suggests that we achieve similar fast learning rates for polynomially decaying step size with $\theta \in [1/2,1)$. While for the varying $T \asymp  (nd^{-1})^{\frac{1}{2-2\theta}}$, the optimal iterative number should be chosen with $\theta = 1/2$ or closing to $1/2$. We compare Theorem \ref{theorem8} with the most related work \cite{lei2021generalization}.  \cite{lei2021generalization} also studies SGD of nonconvex pairwise learning, and provide $\mathcal{O}(n^{-\frac{1}{2}\log^2(1/\delta)}(d+\log(1/\delta))^{\frac{1}{2}})$ order rates, which has the same order $\mathcal{O}(\sqrt{\frac{d}{n}})$ as ours. However, the proof technique between Theorem \ref{theorem8} and \cite{lei2021generalization} is different. Another difference is that \cite{lei2021generalization} studies the case $\eta_t = \eta/\sqrt{T}$ with $\eta \leq \sqrt{T}/(2T)$, while Theorem \ref{theorem8} studies with different step sizes. Theorem \ref{theorem8} is thus served as an important complementary result for nonconvex pairwise learning. 
\end{remark}
\begin{theorem}\label{theorem9}
Suppose Assumptions \ref{assum2}, \ref{assum4}  and \ref{assum5} hold, and
the objective function $f$ is nonnegative, and suppose empirical risk $F_S$ satisfies Assumption \ref{assum3} with parameter $\mu$.
Let $\{ \mathbf{w}_t\}_t$ be the sequence produced by Algorithm \ref{alg:example} with $\eta_t = \eta_1 t^{- \theta} , \theta \in (0,1)$ and $\eta_1 \leq \frac{1}{2\beta}$. Then,
for any $\delta >0$, with probability at least $1 - \delta$ over the sample $S$, choosing $T \asymp n^{\frac{2}{\theta}}$ if $\theta \in (0,1/2)$, $T \asymp  n^4$ if $\theta = 1/2$ and $T \asymp n^{\frac{2}{1-\theta}}  $ if $\theta \in (1/2,1)$, when $n \geq \frac{c\beta^2(d+ \log(\frac{16 \log(\sqrt{2}n R_1 +1)}{\delta}))}{\mu^2}$, we have
\begin{align*}
F(\mathbf{w}_{T+1}) - F(\mathbf{w}^{\ast}) 
= \mathcal{O} \Big(\frac{\log^2(1/\delta)}{n^2}  + \frac{ F(\mathbf{w}^{\ast}) \log(1/\delta)}{ n} \Big).
\end{align*}
If further assume $F(\mathbf{w}^{\ast}) = \mathcal{O} \Big(\frac{1}{n} \Big)$, we have 
\begin{align*}
F(\mathbf{w}_{T+1}) - F(\mathbf{w}^{\ast})   = \mathcal{O} \Big( \frac{\log^2(1/\delta)}{n^2}  \Big).
\end{align*}
\end{theorem}
\begin{remark}\rm{}
Theorem \ref{theorem9} reveals that under the PL condition, the learning rate of SGD can be significantly improved compared to Theorem \ref{theorem8}.  
In the related work,
if $f$ is nonnegative, Lipschitz continuous and smooth, $F_S$ satisfies the PL condition, and Assumption \ref{assum5} hold, the learning rate derived for $\mathbb{E}[F(\mathbf{w}_{T+1}) - F(\mathbf{w}^{\ast})] $ in \cite{lei2021generalization} is at most of order $\mathcal{O}(n^{-\frac{2}{3}})$. By a comparison, one can see that our learning rates are derived with high probability and are significantly faster than the results in \cite{lei2021generalization}. The generalization performance gap is also studied for pairwise SGD in \cite{lei2020sharper} via algorithmic stability. However, their learning rate is limited to convex learning. Specifically, if $f$ is convex and smooth, $F(\mathbf{w}_{T+1}) - F(\mathbf{w}^{\ast})$ is of order $\mathcal{O}(\log n \sqrt{T}/n + n^{-\frac{1}{2}}) + \mathcal{O}(T^{-\frac{1}{2}}\log T)$. By taking the optimal $T \asymp n$, the learning rate becomes $\mathcal{O}(n^{-\frac{1}{2} }\log n)$, which is much slower than results of Theorem \ref{theorem9}. To our best knowledge, the $\mathcal{O}(\frac{1}{n})$ rate is the first for SGD in nonconvex pairwise learning, and the $\mathcal{O}(\frac{1}{n^2})$ rate is also the first whether in convex or nonconvex pairwise learning. 
\end{remark}

\begin{theorem}\label{theorem10}
Suppose Assumptions \ref{assum2}, \ref{assum4}  and \ref{assum5} hold, and
the objective function $f$ is nonnegative, and suppose empirical risk $F_S$ satisfies Assumption \ref{assum3} with parameter $2\mu$.
Let $\{ \mathbf{w}_t\}_t$ be the sequence produced by Algorithm \ref{alg:example} with $\eta_t = \frac{2}{\mu (t+t_0)}$ such that $t_0 \geq \max \left\{\frac{4\beta}{\mu},1 \right\}$ for all $t \in \mathbb{N}$. Then,
for any $\delta >0$, with probability at least $1 - \delta$ over the sample $S$, when $n \geq \frac{c\beta^2(d+ \log(\frac{16 \log(\sqrt{2}n R_1 +1)}{\delta}))}{\mu^2}$ and $T \asymp n^2$, we have
\begin{align*}
F(\mathbf{w}_{T+1}) - F(\mathbf{w}^{\ast}) 
= \mathcal{O} \Big(\frac{\log n \log^3(\frac{1}{\delta})}{n^2} + \frac{ F(\mathbf{w}^{\ast}) \log\frac{1}{\delta}}{ n} \Big).
\end{align*}
If further assume $F(\mathbf{w}^{\ast}) = \mathcal{O} \left(\frac{1}{n} \right)$, we have 
\begin{align*}
F(\mathbf{w}_{T+1}) - F(\mathbf{w}^{\ast})   = \mathcal{O} \Big( \frac{\log n \log^3(1/\delta)}{n^2}  \Big).
\end{align*}
\end{theorem}
\begin{remark}\rm{}
Theorem \ref{theorem10} improves Theorem \ref{theorem9} with a better iterative complexity when selecting the special step size. If we take $T \asymp n$, the learning rate of the generalization performance gap is of order $\frac{\log n \log^3(\frac{1}{\delta})}{n}$, which is still faster than the existing rates in the related work, as discussed before. Additionally, please refer to Table \ref{tab1} in Appendix \ref{appendix2} for a systematic comparison with the related work.
\end{remark}

\section{Proofs}\label{Section4}
In this section, we provide proofs of theorems in Section 3.
\subsection{Proof of Remark 1}\label{section4.1}
\begin{proof}
According to the uniform smoothness condition, for any sample $z,z' \in \mathcal{Z}$ and $\mathbf{w}_1, \mathbf{w}_2 \in \mathcal{W}$, there holds
\begin{align*}
\|\nabla f(\mathbf{w}_1;z,z') - \nabla f(\mathbf{w}_2;z,z') \| \leq \beta \| \mathbf{w}_1 - \mathbf{w}_2 \|.
\end{align*}
Then, for any unit vector $\mathbf{u} \in B(0,1)$, 
we have 
\begin{align*}
&|\mathbf{u}^T (\nabla f(\mathbf{w}_1;z,z') - \nabla f(\mathbf{w}_2;z,z') )|  \\\leq &\| \mathbf{u}\| \|\nabla f(\mathbf{w}_1;z,z') - \nabla f(\mathbf{w}_2;z,z') \| 
\leq \beta \| \mathbf{w}_1 - \mathbf{w}_2 \|,
\end{align*}
which implies
\begin{align*}
\frac{|\mathbf{u}^T (\nabla f(\mathbf{w}_1;z,z') - \nabla f(\mathbf{w}_2;z,z'))|}{\beta \| \mathbf{w}_1 - \mathbf{w}_2 \|} \leq 1. 
\end{align*}
Then we get
\begin{align*}
\mathbb{E} \Big \{\exp \Big (\frac{\ln 2 |\mathbf{u}^T (\nabla f(\mathbf{w}_1;z,z') - \nabla f(\mathbf{w}_2;z,z'))|}{\beta \| \mathbf{w}_1 - \mathbf{w}_2 \|} \Big ) \Big \}\leq 2. 
\end{align*}
So we obtain that $\frac{ \nabla f(\mathbf{w}_1;z,z') - \nabla f(\mathbf{w}_2;z,z') }{\| \mathbf{w}_1 - \mathbf{w}_2 \|} $ is a $\frac{\beta}{\ln 2}$-sub-exponential random vector, for all $\mathbf{w}_1, \mathbf{w}_2 \in \mathcal{W}$.

Furthermore, when Assumption \ref{assum1} holds, according to Jensen's inequality, we can derive that
\begin{align*}
\exp  \Big \{\mathbb{E}\Big (\frac{|\mathbf{u}^T (\nabla f(\mathbf{w}_1;z,z') - \nabla f(\mathbf{w}_2;z,z'))|}{\beta \| \mathbf{w}_1 - \mathbf{w}_2 \|} \Big ) \Big \}\leq 2,
\end{align*}
which means 
\begin{align*}
\mathbb{E} \| \nabla f(\mathbf{w}_1;z,z') - \nabla f(\mathbf{w}_2;z,z')  \| \leq \beta \| \mathbf{w}_1 - \mathbf{w}_2 \|.
\end{align*}
Further by Jensen's inequality, we obtain 
\begin{align*}
\| \nabla F(\mathbf{w}_1) - \nabla F(\mathbf{w}_2)  \| \leq \beta \| \mathbf{w}_1 - \mathbf{w}_2 \|.
\end{align*}
The proof is complete.
\end{proof}
\subsection{Proof of Theorem \ref{theorem1}}\label{proof4.2}
The proof is inspired by the recent breakthrough work \cite{xu2020toward}.
To prove Theorem \ref{theorem1}, we need many preliminaries on generic chaining and two more general forms of the Bernstein inequality of pairwise learning. Considering the length limit, we leave the introduction of this part to Appendix \ref{appendix1}. 
\begin{proof}
We define $\mathcal{V} = \left\{ \mathbf{v} \in \mathbb{R}^d : \| \mathbf{v} \| \leq \max \{ R_1, \frac{1}{n} \} \right\}$.
For all $(\mathbf{w},\mathbf{v}) \in \mathcal{W} \times \mathcal{V}$, let $g_{(\mathbf{w},\mathbf{v})} = (\nabla f(\mathbf{w};z,z') - \nabla f(\mathbf{w}^{\ast};z,z'))^T \mathbf{v}$. Also, for any $(\mathbf{w}_1,\mathbf{v}_1)$ and $(\mathbf{w}_2,\mathbf{v}_2) \in \mathcal{W} \times \mathcal{V}$, we define the following norm on the product space $\mathcal{W} \times \mathcal{V}$,
\begin{align*}
\| (\mathbf{w}_1,\mathbf{v}_1) - (\mathbf{w}_2,\mathbf{v}_2)  \|_{\mathcal{W} \times \mathcal{V}} = (\|\mathbf{w}_1 -\mathbf{w}_2   \|^2 + \| \mathbf{v}_1 - \mathbf{v}_2  \|^2)^{\frac{1}{2}}.
\end{align*}
Define a ball $B(\sqrt{r}) = \{ (\mathbf{w},\mathbf{v}) \in \mathcal{W} \times \mathcal{V}: \| \mathbf{w} - \mathbf{w}^{\ast} \|^2 + \| \mathbf{v}  \|^2 \leq r \}$. Given any $(\mathbf{w}_1,\mathbf{v}_1)$ and $(\mathbf{w}_2,\mathbf{v}_2) \in B(\sqrt{r})$, we make the following decomposition
\begin{align*}
&g_{(\mathbf{w}_1,\mathbf{v}_1)}(z,z') - g_{(\mathbf{w}_2,\mathbf{v}_2)}(z,z')\\
= &(\nabla f(\mathbf{w}_1;z,z') - \nabla f(\mathbf{w}^{\ast};z,z'))^T \mathbf{v}_1 \\
&- (\nabla f(\mathbf{w}_2;z,z') - \nabla f(\mathbf{w}^{\ast};z,z'))^T \mathbf{v}_2 \\
= &(\nabla f(\mathbf{w}_1;z,z') - \nabla f(\mathbf{w}^{\ast};z,z'))^T (\mathbf{v}_1 - \mathbf{v}_2) \\
&+(\nabla f(\mathbf{w}_1;z,z') - \nabla f(\mathbf{w}^{\ast};z,z'))^T \mathbf{v}_2\\
&- (\nabla f(\mathbf{w}_2;z,z') - \nabla f(\mathbf{w}^{\ast};z,z'))^T \mathbf{v}_2 \\
= &(\nabla f(\mathbf{w}_1;z,z') - \nabla f(\mathbf{w}^{\ast};z,z'))^T (\mathbf{v}_1 - \mathbf{v}_2) \\
&+ (\nabla f(\mathbf{w}_1;z,z') - \nabla f(\mathbf{w}_2;z,z'))^T \mathbf{v}_2.
\end{align*}
Since $(\mathbf{w}_1,\mathbf{v}_1)$ and $(\mathbf{w}_2,\mathbf{v}_2) \in B(\sqrt{r})$, there holds that
\begin{align}\label{ineq354567}\nonumber
\| \mathbf{w}_1 -  \mathbf{w}^{\ast}\|  \| \mathbf{v}_1 - \mathbf{v}_2  \| &\leq \sqrt{r}\| \mathbf{v}_1 - \mathbf{v}_2 \| \\
& \leq \sqrt{r} \|  (\mathbf{w}_1,\mathbf{v}_1) - (\mathbf{w}_2,\mathbf{v}_2) \|_{\mathcal{W} \times \mathcal{V}}.
\end{align}
And, according to Assumption \ref{assum1}, we know that $\frac{\nabla f(\mathbf{w}_1,z,z') - \nabla f(\mathbf{w}_2,z,z')}{\| \mathbf{w}_1 - \mathbf{w}_2\|}$ is a $\gamma$-sub-exponential random vector for all $\mathbf{w}_1, \mathbf{w}_2 \in \mathcal{W}$, which means that
\begin{align}\label{ineq3545}
\mathbb{E} \Big\{ \exp \Big( \frac{(\nabla f(\mathbf{w}_1;z,z') - \nabla f(\mathbf{w}^{\ast};z,z'))^T (\mathbf{v}_1 - \mathbf{v}_2)}{\gamma \| \mathbf{w}_1 -  \mathbf{w}^{\ast}\|  \| \mathbf{v}_1 - \mathbf{v}_2  \|} \Big) \Big\}\leq 2.
\end{align}
Now, combined with (\ref{ineq3545}) and (\ref{ineq354567}), and according to Definition \ref{defi1} of Appendix \ref{appendix1}, we know
$(\nabla f(\mathbf{w}_1;z,z') - \nabla f(\mathbf{w}^{\ast};z,z'))^T (\mathbf{v}_1 - \mathbf{v}_2)$ is $\gamma \sqrt{r} \|  (\mathbf{w}_1,\mathbf{v}_1) - (\mathbf{w}_2,\mathbf{v}_2) \|_{\mathcal{W} \times \mathcal{V}}$-sub-exponential. Similarly, we can derive that
\begin{align*}
\| \mathbf{w}_1 -  \mathbf{w}_2\|  \|  \mathbf{v}_2  \| &\leq \sqrt{r}\|\mathbf{w}_1 -  \mathbf{w}_2\| \\
& \leq \sqrt{r} \|  (\mathbf{w}_1,\mathbf{v}_1) - (\mathbf{w}_2,\mathbf{v}_2) \|_{\mathcal{W} \times \mathcal{V}}.
\end{align*}
Also, there holds that 
\begin{align*}
\mathbb{E} \Big\{ \exp \Big( \frac{(\nabla f(\mathbf{w}_1;z,z') - \nabla f(\mathbf{w}_2;z,z'))^T ( \mathbf{v}_2)}{\gamma \| \mathbf{w}_1 -  \mathbf{w}_2\|  \| \mathbf{v}_2  \|} \Big) \Big\}\leq 2.
\end{align*}
Thus, we know $(\nabla f(\mathbf{w}_1;z,z') - \nabla f(\mathbf{w}_2;z,z'))^T \mathbf{v}_2$ is also $\gamma \sqrt{r} \|  (\mathbf{w}_1,\mathbf{v}_1) - (\mathbf{w}_2,\mathbf{v}_2) \|_{\mathcal{W} \times \mathcal{V}}$-sub-exponential.

Till here, 
for any $(\mathbf{w}_1,\mathbf{v}_1)$ and $(\mathbf{w}_2,\mathbf{v}_2) \in B(\sqrt{r})$, we obtain 
\begin{align}\label{ineryu697898}\nonumber
&\mathbb{E} \Big\{\exp\Big( \frac{g_{(\mathbf{w}_1,\mathbf{v}_1)}(z,z') - g_{(\mathbf{w}_2,\mathbf{v}_2)}(z,z')}{2\gamma \sqrt{r} \|  (\mathbf{w}_1,\mathbf{v}_1) - (\mathbf{w}_2,\mathbf{v}_2) \|_{\mathcal{W} \times \mathcal{V}}} \Big) \Big\}\\\nonumber
\leq &\mathbb{E} \Big\{\frac{1}{2} \exp\Big( \frac{(\nabla f(\mathbf{w}_1;z,z') - \nabla f(\mathbf{w}^{\ast};z,z'))^T (\mathbf{v}_1 - \mathbf{v}_2)}{\gamma \sqrt{r} \|  (\mathbf{w}_1,\mathbf{v}_1) - (\mathbf{w}_2,\mathbf{v}_2) \|_{\mathcal{W} \times \mathcal{V}}} \Big) \Big \}\\
 + &\mathbb{E} \Big\{\frac{1}{2} \exp\Big( \frac{(\nabla f(\mathbf{w}_1;z,z') - \nabla f(\mathbf{w}_2;z,z'))^T ( \mathbf{v}_2)}{\gamma \sqrt{r} \|  (\mathbf{w}_1,\mathbf{v}_1) - (\mathbf{w}_2,\mathbf{v}_2) \|_{\mathcal{W} \times \mathcal{V}}} \Big) \Big\} \leq 2,
\end{align}
where the first inequality follows from Jensen's inequality.
And (\ref{ineryu697898}) means that $g_{(\mathbf{w}_1,\mathbf{v}_1)}(z,z') - g_{(\mathbf{w}_2,\mathbf{v}_2)}(z,z')$ is a $2\gamma \sqrt{r} \|  (\mathbf{w}_1,\mathbf{v}_1) - (\mathbf{w}_2,\mathbf{v}_2) \|_{\mathcal{W} \times \mathcal{V}}$-sub-exponential random variable,
 that is 
\begin{align}\label{inru4it05869}\nonumber
&\| g_{(\mathbf{w}_1,\mathbf{v}_1)}(z,z') - g_{(\mathbf{w}_2,\mathbf{v}_2)}(z,z') \|_{Orlicz-1} \\
\leq &2\gamma \sqrt{r} \|  (\mathbf{w}_1,\mathbf{v}_1) - (\mathbf{w}_2,\mathbf{v}_2) \|_{\mathcal{W} \times \mathcal{V}}.
\end{align}

Then, the next step is to apply the Bernstein inequality of pairwise learning (Lemma \ref{lemma5} of Appendix \ref{appendix1}) to $g_{(\mathbf{w}_1,\mathbf{v}_1)}(z,z') - g_{(\mathbf{w}_2,\mathbf{v}_2)}(z,z')$. From (\ref{inru4it05869}), we know that the Bernstein parameters of sub-exponential $g_{(\mathbf{w}_1,\mathbf{v}_1)}(z,z') - g_{(\mathbf{w}_2,\mathbf{v}_2)}(z,z')$ are $2\gamma \sqrt{r} \|  (\mathbf{w}_1,\mathbf{v}_1) - (\mathbf{w}_2,\mathbf{v}_2) \|_{\mathcal{W} \times \mathcal{V}}$ (see Lemma \ref{lemma8} of Appendix \ref{appendix1}). Now, we can derive that 
\begin{align}\nonumber\label{eq7}
&Pr\Big(\Big|(P-P_n)[g_{(\mathbf{w}_1,\mathbf{v}_1)}(z,z') - g_{(\mathbf{w}_2,\mathbf{v}_2)}(z,z')] \Big|\\\nonumber
\geq & 2\gamma \sqrt{r} \|  (\mathbf{w}_1,\mathbf{v}_1) - (\mathbf{w}_2,\mathbf{v}_2) \|_{\mathcal{W} \times \mathcal{V}} \sqrt{\frac{2u}{\lfloor \frac{n}{2} \rfloor}} \\ &+ \frac{2\gamma \sqrt{r} \|  (\mathbf{w}_1,\mathbf{v}_1) - (\mathbf{w}_2,\mathbf{v}_2) \|_{\mathcal{W} \times \mathcal{V}}}{\lfloor \frac{n}{2} \rfloor}u \Big) \leq 2e^{-u},
\end{align}
where $\lfloor \frac{n}{2} \rfloor $ is the
largest integer no greater than $\frac{n}{2}$ and "Pr" means probability.
According to Definition \ref{defi3} of Appendix \ref{appendix1}, (\ref{eq7}) implies that the process $(P-P_n)[g_{(\mathbf{w},\mathbf{v})}(z,z') ]$ has a mixed sub-Gaussian-sub-exponential increments w.r.t. the metric pair $\Big(\frac{2\gamma \sqrt{r} \| \cdot\|_{\mathcal{W} \times \mathcal{V}}}{\lfloor \frac{n}{2} \rfloor}, 2\gamma\| \cdot\|_{\mathcal{W} \times \mathcal{V}}\sqrt{\frac{ 2r }{\lfloor \frac{n}{2} \rfloor}} \Big)$. Hence, from the generic chaining for a process with mixed tail increments in Lemma \ref{lemma2} of Appendix \ref{appendix1}, for all $\delta \in (0,1)$, with probability at least $1-\delta$, we have
\begin{align*}
&\sup_{\|\mathbf{w} - \mathbf{w}^{\ast}\|^2 + \| \mathbf{v} \|^2 \leq r} | (P-P_n)[g_{(\mathbf{w},\mathbf{v})}(z,z')] | \\
\leq &C \Big( \gamma_{2} \Big(B(\sqrt{r}),2\gamma\| \cdot\|_{\mathcal{W} \times \mathcal{V}}\sqrt{\frac{ 2r }{\lfloor \frac{n}{2} \rfloor}} \Big)\\
&+ \gamma_{1} \Big(B(\sqrt{r}),\frac{2\gamma \sqrt{r} \| \cdot\|_{\mathcal{W} \times \mathcal{V}}}{\lfloor \frac{n}{2} \rfloor} \Big)  + \gamma r \frac{\log \frac{1}{\delta}}{\lfloor \frac{n}{2} \rfloor}+
 \gamma r\sqrt{\frac{\log \frac{1}{\delta}}{\lfloor \frac{n}{2} \rfloor}}    \Big).
\end{align*}
From Lemma \ref{lemma1.5} of Appendix \ref{appendix1}, the $\gamma_1$ functional and the $\gamma_2$ functional can be bounded by the Dudley's integral, which implies that there exists an absolute constant $C$ such that for any $\delta \in (0,1)$, with probability at least $1-\delta$
\begin{align}\label{iejfiehjt984}\nonumber
&\sup_{\|\mathbf{w} - \mathbf{w}^{\ast}\|^2 + \| \mathbf{v} \|^2 \leq r} | (P-P_n)[g_{(\mathbf{w},\mathbf{v})}(z,z')] | \\
\leq & C\gamma r \Big(\sqrt{\frac{d+\log \frac{1}{\delta}}{\lfloor \frac{n}{2} \rfloor}} + \frac{d+\log \frac{1}{\delta}}{\lfloor \frac{n}{2} \rfloor} \Big),
\end{align}
where the inequality follows from (B.3) of \cite{xu2020toward}. Till here,
the next step is to apply Lemma \ref{lemma1} of Appendix \ref{appendix1} to (\ref{iejfiehjt984}). 

We set $T(f) = \|\mathbf{w} - \mathbf{w}^{\ast}\|^2 + \| \mathbf{v} \|^2 $,  $\psi(r;\delta) = C\gamma r \Big(\sqrt{\frac{d+\log \frac{1}{\delta}}{\lfloor \frac{n}{2} \rfloor}} + \frac{d+\log \frac{1}{\delta}}{\lfloor \frac{n}{2} \rfloor} \Big)$. Since $\|\mathbf{w} - \mathbf{w}^{\ast}\|^2 + \| \mathbf{v} \|^2 \leq R_1^2 + R_1^2 + \frac{1}{n^2}$, we set $R = 2R_1^2 + \frac{1}{n^2}$. And let $r_0 = \frac{2}{n^2}$. Applying Lemma \ref{lemma1}, we obtain that for any $\delta \in (0,1)$, with probability at least $1-\delta$, for all $\mathbf{w} \in \mathcal{W}$ and $\mathbf{v} \in \mathcal{V}$,
\begin{align}\nonumber\label{ineq18}
&(P-P_n)[g_{(\mathbf{w},\mathbf{v})}(z,z')] \\\nonumber
=& (P-P_n) [(\nabla f(\mathbf{w};z,z') - \nabla f(\mathbf{w}^{\ast};z,z'))^T \mathbf{v}]\\\nonumber
\leq & \psi \Big(\max \Big\{\|\mathbf{w} - \mathbf{w}^{\ast}\|^2 + \| \mathbf{v} \|^2  ,\frac{2}{n^2}\Big\};\frac{\delta}{2 \log_2(R n^2)} \Big)\\\nonumber
= & C\gamma \max \Big\{\|\mathbf{w} - \mathbf{w}^{\ast}\|^2 + \| \mathbf{v} \|^2  ,\frac{2}{n^2}\Big\} \\
&\times \Big( \sqrt{\frac{d+\log \frac{2 \log_2(R n^2)}{\delta}}{\lfloor \frac{n}{2} \rfloor}} + \frac{d+\log \frac{2 \log_2(R n^2)}{\delta}}{\lfloor \frac{n}{2} \rfloor}  \Big).
\end{align}

Now, we 
choose $\mathbf{v}$ as
$\max \{\|\mathbf{w} - \mathbf{w}^{\ast}\|, \frac{1}{n}  \}\frac{(P-P_n) (\nabla f(\mathbf{w};z,z') - \nabla f(\mathbf{w}^{\ast};z,z'))}{\|  (P-P_n) (\nabla f(\mathbf{w};z,z') - \nabla f(\mathbf{w}^{\ast};z,z')) \|}$.
It is clear that $\|\mathbf{v} \| = \max \{\|\mathbf{w} - \mathbf{w}^{\ast}\|, \frac{1}{n}  \} \leq \max \{R_1, \frac{1}{n}  \}$, which belongs to the space $\mathcal{V}$.
Plugging this $\mathbf{v}$ into (\ref{ineq18}),
 we obtain that for any $\delta \in (0,1)$, with probability at least $1-\delta$, for all $\mathbf{w} \in \mathcal{W}$,
\begin{align}\label{inehru4r9}\nonumber
&\|(P-P_n) (\nabla f(\mathbf{w};z,z') - \nabla f(\mathbf{w}^{\ast};z,z')) \|\\\nonumber
\leq & C\gamma \max \Big\{\|\mathbf{w} - \mathbf{w}^{\ast}\|  ,\frac{1}{n}\Big\} \\\nonumber
&\times \Big( \sqrt{\frac{d+\log \frac{2 \log_2(R n^2)}{\delta}}{\lfloor \frac{n}{2} \rfloor}} + \frac{d+\log \frac{2 \log_2(R n^2)}{\delta}}{\lfloor \frac{n}{2} \rfloor}  \Big)\\\nonumber
\leq &C\gamma \max \Big\{\|\mathbf{w} - \mathbf{w}^{\ast}\|  ,\frac{1}{n}\Big\}\\&\times \Big( \sqrt{\frac{d+\log \frac{2 \log_2(R n^2)}{\delta}}{n}} + \frac{d+\log \frac{2 \log_2(R n^2)}{\delta}}{n}  \Big).
\end{align}
Since $R = 2R_1^2 + \frac{1}{n^2}$, (\ref{inehru4r9}) thus implies that
\begin{align*}
&\|(P-P_n) (\nabla f(\mathbf{w};z,z') - \nabla f(\mathbf{w}^{\ast};z,z')) \|\\
\leq&C\gamma \max \Big\{\|\mathbf{w} - \mathbf{w}^{\ast}\| ,\frac{1}{n}\Big\} \\
&\times \Big( \sqrt{\frac{d+\log \frac{4 \log_2(\sqrt{2}R_1 n + 1)}{\delta}}{n}} + \frac{d+\log \frac{4 \log_2(\sqrt{2}R_1 n + 1)}{\delta}}{n} \Big ).
\end{align*}
The proof is complete.
\end{proof}
\subsection{Proof of Theorem \ref{theorem2}}
\begin{proof}
From Theorem \ref{theorem1}, we have
\begin{align}\label{ineq178}\nonumber
&\|\nabla F (\mathbf{w} )-\nabla F_S(\mathbf{w})\| \\\nonumber\leq &\|\nabla F (\mathbf{w}^{\ast})-\nabla F_S(\mathbf{w}^{\ast}) \|
+ C\gamma \max \Big\{\|\mathbf{w} - \mathbf{w}^{\ast}\|  ,\frac{1}{n}\Big\} \\
\times &\Big( \sqrt{\frac{d+\log \frac{4 \log_2(\sqrt{2}R_1 n + 1)}{\delta}}{n}} + \frac{d+\log \frac{4 \log_2(\sqrt{2}R_1 n + 1)}{\delta}}{n}  \Big),
\end{align}
where the inequality follows from that $\|\nabla F (\mathbf{w} )-\nabla F_S(\mathbf{w})\| - \|\nabla F (\mathbf{w}^{\ast})-\nabla F_S(\mathbf{w}^{\ast}) \| \leq \|(\nabla F (\mathbf{w} )-\nabla F_S(\mathbf{w}))  -  (\nabla F (\mathbf{w}^{\ast})-\nabla F_S(\mathbf{w}^{\ast})) \|$.
Denote  $\xi_{n,R_1,d,\delta} = \sqrt{\frac{d+\log \frac{4 \log_2(\sqrt{2}R_1 n + 1)}{\delta}}{n}} + \frac{d+\log \frac{4 \log_2(\sqrt{2}R_1 n + 1)}{\delta}}{n}$. We are now to prove the bound of $\|\nabla F (\mathbf{w}^{\ast})-\nabla F_S(\mathbf{w}^{\ast})   \|$. 

It is clear that $\nabla F (\mathbf{w}^{\ast}) = 0$. From Lemma \ref{lemma7} of Appendix \ref{appendix1} and Assumption \ref{assum2}, we have the following inequality for any $\delta > 0$, with probability at least $1-\delta$
\begin{align}\label{ineq188}\nonumber
&\|\nabla F (\mathbf{w}^{\ast})-\nabla F_S(\mathbf{w}^{\ast})  \|\\
\leq & \sqrt{\frac{2\mathbb{E}[\| \nabla f(\mathbf{w}^{\ast};z,z')  \|^2] \log \frac{2}{\delta}}{\lfloor \frac{n}{2} \rfloor}} + \frac{D_{\ast} \log \frac{2}{\delta}}{\lfloor \frac{n}{2} \rfloor}.
\end{align}
Plugging (\ref{ineq188}) into (\ref{ineq178}), we obtain that for any $\delta > 0$, with probability at least $1-\delta$
\begin{align*}
&\|\nabla F (\mathbf{w} )-\nabla F_S(\mathbf{w}) \| \leq C\gamma \max \Big\{\|\mathbf{w} - \mathbf{w}^{\ast}\|  ,\frac{1}{n} \Big\} \xi_{n,R_1,d,\frac{\delta}{2}}\\
+&\sqrt{\frac{2\mathbb{E}[\| \nabla f(\mathbf{w}^{\ast};z,z')  \|^2] \log \frac{4}{\delta}}{\lfloor \frac{n}{2} \rfloor}}  + \frac{D_{\ast} \log \frac{4}{\delta}}{\lfloor \frac{n}{2} \rfloor}\\
\leq&\sqrt{\frac{8\mathbb{E}[\| \nabla f(\mathbf{w}^{\ast};z,z')  \|^2] \log \frac{4}{\delta}}{n}} + \frac{4D_{\ast} \log \frac{4}{\delta}}{n}\\
& +  C\gamma \max \Big\{\|\mathbf{w} - \mathbf{w}^{\ast}\|  ,\frac{1}{n}\Big\} \xi_{n,R_1,d,\frac{\delta}{2}}.
\end{align*}
The proof is complete.
\end{proof}
\subsection{Proof of Theorem \ref{theorem3}}
\begin{proof}
Denote  $\xi_{n,R_1,d,\delta} = \sqrt{\frac{d+\log \frac{8 \log_2(\sqrt{2}R_1 n + 1)}{\delta}}{n}} + \frac{d+\log \frac{8 \log_2(\sqrt{2}R_1 n + 1)}{\delta}}{n}$.
According to Theorem \ref{theorem2}, for any $\delta \in (0,1)$, with probability at least $1-\delta$, we have the following inequality
\begin{align}\label{ineq777}\nonumber
&\|\nabla F(\mathbf{w})  -\nabla F_S(\mathbf{w})  \| \leq\sqrt{\frac{8\mathbb{E}[\| \nabla f(\mathbf{w}^{\ast};z,z')  \|^2] \log \frac{4}{\delta}}{n}}\\
& + \frac{4D_{\ast} \log \frac{4}{\delta}}{n} + C\gamma \max \Big\{\|\mathbf{w} - \mathbf{w}^{\ast}\| ,\frac{1}{n} \Big\} \xi_{n,R_1,d,\delta}.
\end{align}
This implies that 
\begin{align*}
&\| \nabla F(\mathbf{w})  \| - \| \nabla F_S(\mathbf{w})  \| \leq C\gamma \max \Big\{\|\mathbf{w} - \mathbf{w}^{\ast}\| ,\frac{1}{n} \Big\} \xi_{n,R_1,d,\delta}\\
&+ \frac{4D_{\ast} \log \frac{4}{\delta}}{n}+\sqrt{\frac{8\mathbb{E}[\| \nabla f(\mathbf{w}^{\ast};z,z')  \|^2] \log \frac{4}{\delta}}{n}}.
\end{align*}
According to Remark \ref{remark1}, Assumption \ref{assum1} implies the population risk $F(\mathbf{w})$ is $\gamma$-smooth. Moreover, when $F(\mathbf{w})$ is smooth and satisfies the PL condition, there holds the following error bound property (refer to Theorem 2 in \cite{karimi2016linear})
\begin{align*}
\|\nabla F(\mathbf{w})  \| \geq \mu \| \mathbf{w} - \mathbf{w}^{\ast} \|.
\end{align*}
Thus, we have
\begin{align}\label{rignoirgj}\nonumber
&\mu \| \mathbf{w} - \mathbf{w}^{\ast} \| \leq \|\nabla F(\mathbf{w})  \| \leq \| \nabla F_S(\mathbf{w})  \|\\\nonumber
& + \sqrt{\frac{8\mathbb{E}[\| \nabla f(\mathbf{w}^{\ast};z,z')  \|^2] \log \frac{4}{\delta}}{n}} + \frac{4D_{\ast} \log \frac{4}{\delta}}{n}\\
& + C\gamma \max \Big\{\|\mathbf{w} - \mathbf{w}^{\ast}\| ,\frac{1}{n} \Big\} \xi_{n,R_1,d,\delta}.
\end{align}
And according to \cite{nesterov2014introductory}, there holds the following property for $\gamma$-smooth functions $f$:
\begin{align}\label{09ur90hgoij}
\frac{1}{2\gamma} \| \nabla f(\mathbf{w})  \|^2 \leq f(\mathbf{w}) - \inf_{\mathbf{w} \in \mathcal{W}} f(\mathbf{w}).
\end{align}
Thus we have 
\begin{align}\label{kljge90g90}
\frac{1}{2\gamma} \| \nabla F(\mathbf{w})  \|^2 \leq F(\mathbf{w}) - F(\mathbf{w}^{\ast}) \leq \frac{\left\| \nabla F(\mathbf{w}) \right\|^2 }{2 \mu},
\end{align}
which means that $\frac{\mu}{\gamma} \leq 1$.
Let $c = \max\{ 4{C}^2, 1\}$. When 
\begin{align*}
n \geq \frac{c\gamma^2(d+ \log \frac{8 \log_2(\sqrt{2}R_1 n + 1)}{\delta})}{\mu^2},
\end{align*}
we have $C \gamma \xi_{n,R_1,d,\delta}\leq \frac{\mu}{2}$, followed from the fact that $\frac{\mu}{\gamma} \leq 1$. 

Plugging $C \gamma \xi_{n,R_1,d,\delta}\leq \frac{\mu}{2}$ into (\ref{rignoirgj}), we can derive that
\begin{align}\label{ineq445}\nonumber
\| \mathbf{w} - \mathbf{w}^{\ast} \| \leq \frac{2}{\mu} \Big(\left\| \nabla F_S(\mathbf{w}) \right\| + \frac{4D_{\ast}\log(4/\delta)}{n}\\ + \sqrt{\frac{8 \mathbb{E} [ \| \nabla f(\mathbf{w}^{\ast};z,z') \|^2 ] \log(4/\delta)}{n}} + \frac{\mu}{2n} \Big).
\end{align}
Then, substituting (\ref{ineq445}) into (\ref{ineq777}), we derive that for all $\mathbf{w} \in \mathcal{W}$, when $n \geq \frac{c\gamma^2(d+ \log \frac{8 \log_2(\sqrt{2}R_1 n + 1)}{\delta})}{\mu^2}$, with probability at least $1-\delta$
\begin{align*}
&\hphantom{{}={}}\left\| \nabla F (\mathbf{w} )- \nabla F_S(\mathbf{w}) \right\|\leq \left\| \nabla F_S(\mathbf{w}) \right\| \\
& +  \frac{\mu}{n} + 2\frac{4D_{\ast}\log(4/\delta)}{n} + 2\sqrt{\frac{8 \mathbb{E} [ \| \nabla f(\mathbf{w}^{\ast};z,z') \|^2 ] \log(4/\delta)}{n}}.
\end{align*}
The proof is complete.
\end{proof}
\subsection{Proof of Theorem \ref{theorem4}}
\begin{proof}
Plugging $\hat{\mathbf{w}}^{\ast}$ into Theorem \ref{theorem2}, we have
\begin{align*}
&\| \nabla F(\hat{\mathbf{w}}^{\ast})  \|- \|\nabla F_S(\hat{\mathbf{w}}^{\ast})  \| \\
\leq&\sqrt{\frac{8\mathbb{E}[\| \nabla f(\mathbf{w}^{\ast};z,z')  \|^2] \log \frac{4}{\delta}}{n}} + \frac{4D_{\ast} \log \frac{4}{\delta}}{n}\\
& + C\gamma \max \Big\{\|\hat{\mathbf{w}}^{\ast} - \mathbf{w}^{\ast}\|  ,\frac{1}{n} \Big\} \\
&\times \Big( \sqrt{\frac{d+\log \frac{8 \log_2(\sqrt{2}R_1 n + 1)}{\delta}}{n}} + \frac{d+\log \frac{8 \log_2(\sqrt{2}R_1 n + 1)}{\delta}}{n} \Big ).
\end{align*}
Since $\hat{\mathbf{w}}^{\ast}$ is the ERM of $F_S$, there holds that $ \nabla F_S(\hat{\mathbf{w}}^{\ast}) = 0$. Thus, we can derive that 
\begin{align*}
&\|\nabla F(\hat{\mathbf{w}}^{\ast})  \| 
\leq \sqrt{\frac{8\mathbb{E}[\| \nabla f(\mathbf{w}^{\ast};z,z')  \|^2] \log \frac{4}{\delta}}{n}} \\ &+ \frac{4D_{\ast} \log \frac{4}{\delta}}{n} + C\gamma \Big(R_1  + \frac{1}{n} \Big) \\
&\times \Big( \sqrt{\frac{d+\log \frac{8 \log_2(\sqrt{2}R_1 n + 1)}{\delta}}{n}} + \frac{d+\log \frac{8 \log_2(\sqrt{2}R_1 n + 1)}{\delta}}{n} \Big ).
\end{align*}
The proof is complete.
\end{proof}
\subsection{Proof of Theorem \ref{theorem5}}

\begin{proof}
Since $F(\mathbf{w})$ satisfies the PL condition with parameter $\mu$, we have
\begin{align*}
F(\mathbf{w}) -F(\mathbf{w}^{\ast}) \leq \frac{\left\| \nabla F(\mathbf{w}) \right\|^2 }{2 \mu}, \quad \forall \mathbf{w} \in \mathcal{W}.
\end{align*}
Therefore, to bound the excess risk $F(\hat{\mathbf{w}}^{\ast}) -F(\mathbf{w}^{\ast})$, we need to bound the term $\left\| \nabla F(\hat{\mathbf{w}}^{\ast}) \right\|^2$.
Plugging $\hat{\mathbf{w}}^{\ast}$ into Theorem \ref{theorem3} and (\ref{iqe23}), 
for any $\delta >0 $, when $n \geq \frac{c\gamma^2(d+ \log \frac{8 \log_2(\sqrt{2}R_1 n + 1)}{\delta})}{\mu^2}$, with probability at least $1- \delta$,
\begin{align*}
&\|  \nabla F(\hat{\mathbf{w}}^{\ast})  \| \leq 2\left\| \nabla F_S(\hat{\mathbf{w}}^{\ast}) \right\| +  \frac{\mu}{n}\\
&+ \frac{8D_{\ast}\log(4/\delta)}{n} + 4\sqrt{\frac{2 \mathbb{E} [ \| \nabla f(\mathbf{w}^{\ast};z,z') \|^2 ] \log(4/\delta)}{n}},
\end{align*}
Since $\nabla F_S(\hat{\mathbf{w}}^{\ast})  = 0$, we have $\left\| \nabla F_S(\hat{\mathbf{w}}^{\ast}) \right\| = 0$. We can derive that
\begin{align*}
&F(\hat{\mathbf{w}}^{\ast}) -F(\mathbf{w}^{\ast}) \\
\leq &  \frac{12 D^2_{\ast}\log^2(4/\delta)}{\mu n^2} + \frac{6 \mathbb{E} [ \| \nabla f(\mathbf{w}^{\ast};z,z') \|^2 \log(4/\delta)}{\mu n} + \frac{2\mu}{n^2}.
\end{align*}
The proof is complete.
\end{proof}
\subsection{Proof of Theorem \ref{theorem6}}
\begin{proof}
According to Assumption \ref{assum4} and $\eta_t \leq  1/\beta$, we can derive that
\begin{align}\label{inew77}\nonumber
&F_S(\mathbf{w}_{t+1}) - F_S(\mathbf{w}_{t}) \\\nonumber
 \leq &\langle \mathbf{w}_{t+1} - \mathbf{w}_{t}, \nabla F_S(\mathbf{w}_{t}) \rangle + \frac{\beta}{2}\| \mathbf{w}_{t+1} -\mathbf{w}_{t}  \|^2\\\nonumber
 = & - \eta_t \| \nabla F_S(\mathbf{w}_{t})  \|^2 + \frac{\beta}{2} \eta_t^2 \| \nabla F_S(\mathbf{w}_{t})   \|^2\\\nonumber
 = & \Big(\frac{\beta}{2} \eta_t^2  - \eta_t \Big)\|  \nabla F_S(\mathbf{w}_{t})  \|^2\\
 \leq & -\frac{1}{2}\eta_t\|  \nabla F_S(\mathbf{w}_{t})  \|^2,
\end{align}
which implies that 
\begin{align*}
\eta_t\| \nabla F_S(\mathbf{w}_{t})  \|^2 \leq -2(F_S(\mathbf{w}_{t+1}) - F_S(\mathbf{w}_{t})).
\end{align*}
Take a summation from $t = 1$ to $T$, we have
\begin{align}\label{klfhjpotojhpo}
\sum_{t = 1}^T \eta_t\| \nabla F_S(\mathbf{w}_{t})  \|^2 \leq 2 (F_S(\mathbf{w}_1) - F_S(\mathbf{w}_{T+1})).
\end{align}
Furthermore, we derive that
\begin{align*}
&\sum_{t = 1}^T \eta_t\| \nabla F(\mathbf{w}_t) \|^2 \\
\leq  &2\sum_{t = 1}^T \eta_t \| \nabla F(\mathbf{w}_t) - \nabla F_S(\mathbf{w}_t) \|^2 + 2\sum_{t = 1}^T \eta_t\| \nabla F_S(\mathbf{w}_t)\|^2 \\
\leq &2\sum_{t = 1}^T  \eta_t\max_{t=1,...,T} \| \nabla F(\mathbf{w}_t) - \nabla F_S(\mathbf{w}_t) \|^2 + \mathcal{O} \left(1\right),
\end{align*}
which implies that with probability at least $1-\delta$
\begin{align}\label{xixixixixixixi}\nonumber
&\frac{1}{\sum_{t = 1}^T \eta_t}\sum_{t = 1}^T \eta_t\| \nabla F(\mathbf{w}_t) \|^2 \\\nonumber
\leq &2  \max_{t=1,...,T}\| \nabla F(\mathbf{w}_t) - \nabla F_S(\mathbf{w}_t) \|^2 + \Big(\sum_{t = 1}^T \eta_t \Big)^{-1}\mathcal{O} \left(1\right)\\\nonumber
\leq &\Big(\sum_{t = 1}^T \eta_t \Big)^{-1}\mathcal{O} \left(1\right) + 2 \max_{t=1,...,T} \Big[C\beta \max \Big\{\|\mathbf{w}_t - \mathbf{w}^{\ast}\|  ,\frac{1}{n} \Big\}  \\
 \times &\Big( \sqrt{\frac{d+\log \frac{4 \log_2(\sqrt{2}R_1 n + 1)}{\delta}}{n}} + \frac{d+\log \frac{4 \log_2(\sqrt{2}R_1 n + 1)}{\delta}}{n}  \Big) \Big]^2,
\end{align}
where  $\mathcal{O}(1)$ in the first inequality is due to (\ref{klfhjpotojhpo}) and the nonnegative property of $f$, and where the second inequality holds since Theorem \ref{theorem2} and that Assumption \ref{assum4} implies Assumption \ref{assum1} (see Remark \ref{remark1}).

We now to prove the bound of $\|\mathbf{w}_t - \mathbf{w}^{\ast}\|$.
Since $\mathbf{w}_1 = 0$ and 
$\mathbf{w}_{t+1} = \mathbf{w}_{t} - \eta_t \nabla F_S(\mathbf{w}_{t})$,
we have 
$\mathbf{w}_{t+1} =  \sum_{k=1}^{t}- \eta_k\nabla F_S(\mathbf{w}_k)$.
And according to Schwartz's inequality, we have
\begin{align*}
 &\Big\| \sum_{k=1}^{t}\eta_k\nabla  F_S(\mathbf{w}_k) \Big\|^2 \leq \Big(\sum_{k=1}^{t}\eta_k\| \nabla F_S(\mathbf{w}_k) \| \Big)^2 \\
\leq &\Big(\sum_{k=1}^{t}\eta_k\Big) \Big(\sum_{k=1}^{t}\eta_k \| \nabla F_S(\mathbf{w}_{k})  \|^2\Big) \leq \Big(\sum_{k=1}^{t}\eta_k\Big) \mathcal{O}(1).
\end{align*}
Then we have
\begin{align*}
&\|  \mathbf{w}_{t+1} - \mathbf{w}^{\ast} \| \leq \|  \mathbf{w}_{t+1}  \| + \| \mathbf{w}^{\ast}\| \\= & \Big\| \sum_{k=1}^{t} \eta_k\nabla F_S(\mathbf{w}_k) \Big\| + \| \mathbf{w}^{\ast}\|=\mathcal{O}\Big( \Big(\sum_{k=1}^{t}\eta_k \Big)^{\frac{1}{2}} \Big).
\end{align*}
If $\theta \in (0,1)$, then $\sum_{k =1}^t k^{-\theta} \leq t^{1-\theta}/(1-\theta)$.
Thus, we have the following result uniformly for all $t =1,...,T$
\begin{align}\label{ineqjfjoiuf}
&\|  \mathbf{w}_{t+1} - \mathbf{w}^{\ast} \| =
             \mathcal{O}\left( T^{\frac{1-\theta}{2}}  \right) &\quad \text{if   } \theta \in (0,1 ).  
\end{align}
Therefore, plugging (\ref{ineqjfjoiuf}) into (\ref{xixixixixixixi}), we get that with probability at least $1-\delta$
\begin{align}\label{jizbnu}\nonumber
&\frac{1}{\sum_{t = 1}^T \eta_t}\sum_{t = 1}^T \eta_t\| \nabla F(\mathbf{w}_t) \|^2
\leq \Big(\sum_{t = 1}^T \eta_t \Big)^{-1}\mathcal{O} \left(1\right) \\\nonumber&+ \mathcal{O} \Big( \frac{d+\log \frac{4 \log_2(\sqrt{2}R_1 n + 1)}{\delta}}{n}  T^{1-\theta} \Big)\\
&\leq \mathcal{O} \Big(\frac{1}{T^{1-\theta}}\Big) + \mathcal{O} \Big(\frac{d + \log \frac{ \log n}{\delta}}{n} T^{1-\theta}  \Big).
\end{align}
If we choose $T \asymp (nd^{-1})^{\frac{1}{2(1-\theta)}}$, then we derive that
\begin{align*}
&\frac{1}{\sum_{t = 1}^T \eta_t}\sum_{t = 1}^T \eta_t\| \nabla F(\mathbf{w}_t) \|^2 
\leq \mathcal{O} \Big(\frac{d^{\frac{1}{2}} + d^{-\frac{1}{2}}\log \frac{ \log n}{\delta}}{n^{\frac{1}{2}}}\Big).
\end{align*}
The proof is complete.
\end{proof}
\subsection{Proof of Theorem \ref{theorem7}}
\begin{proof}
By (\ref{inew77}) and the PL condition of $F_S$, we can prove that
\begin{align*}
&F_S(\mathbf{w}_{t+1}) - F_S(\mathbf{w}_{t}) 
 \leq  -\frac{1}{2} \eta_t \|  \nabla F_S(\mathbf{w}_{t})  \|^2\\
 \leq & -\mu \eta_t (F_S(\mathbf{w}_{t}) - F_S(\hat{\mathbf{w}}^{\ast})),
\end{align*}
which implies that
\begin{align*}
F_S(\mathbf{w}_{t+1}) - F_S(\hat{\mathbf{w}}^{\ast})
 \leq (1-\mu \eta_t) (F_S(\mathbf{w}_{t}) - F_S(\hat{\mathbf{w}}^{\ast})) .
\end{align*}
If $\eta_t \leq \frac{1}{\beta}$, then $0<1-\mu \eta_t <1$ since $\frac{\mu}{\beta} \leq 1$ according to (\ref{kljge90g90}).
Taking over $T$ iterations, we get
\begin{align}\label{xuayu}
F_S(\mathbf{w}_{T+1}) - F_S(\hat{\mathbf{w}}^{\ast}) 
 \leq (1-\mu \eta_t)^{T} (F_S(\mathbf{w}_{1}) - F_S(\hat{\mathbf{w}}^{\ast})).
\end{align}

If $\eta_t = 1/\beta$, combined with (\ref{xuayu}), the smoothness of $F_S$ (see (\ref{09ur90hgoij})), and the nonnegative property of $f$, it can be derived that
\begin{align}\label{gdpl}
\| \nabla F_S(\mathbf{w}_{T+1})  \|^2 =  \mathcal{O}\Big ((1-\frac{\mu}{\beta})^{T} \Big).
\end{align}
Furthermore,
since $F_S$ satisfies the PL assumption with parameter $\mu$, we have
\begin{align}\label{jfhuiohfiooij}
F(\mathbf{w}_{T+1}) -F(\mathbf{w}^{\ast}) \leq \frac{\left\| \nabla F(\mathbf{w}_{T+1}) \right\|^2 }{2 \mu}, \quad \forall  \mathbf{w} \in \mathcal{W}.
\end{align} 
So to bound $F(\mathbf{w}_{T+1}) - F(\mathbf{w}^{\ast}) $, we need to bound the term $\left\| \nabla F(\mathbf{w}_{T+1}) \right\|^2$.
And there holds 
\begin{align}\label{oprjg09rjj}\nonumber
&\left\| \nabla F(\mathbf{w}_{T+1}) \right\|^2 \\
\leq &2 \left\| \nabla F(\mathbf{w}_{T+1})- \nabla F_S(\mathbf{w}_{T+1}) \right\|^2 + 2 \| \nabla F_S(\mathbf{w}_{T+1}) \|^2.
\end{align}
For the first term $\left\| \nabla F(\mathbf{w}_{T+1})- \nabla F_S(\mathbf{w}_{T+1}) \right\|^2$,
from Theorem \ref{theorem3},
for all $\mathbf{w} \in \mathcal{W}$,
 when $n \geq \frac{c\beta^2(d+ \log \frac{8 \log_2(\sqrt{2}R_1 n + 1)}{\delta})}{\mu^2}$, with probability at least $1- \delta$, there holds
\begin{align}\label{ol6k7jupo67uj}\nonumber
&\hphantom{{}={}}\left\| \nabla F (\mathbf{w}_{T+1} )- \nabla F_S(\mathbf{w}_{T+1}) \right\| \leq \left\| \nabla F_S(\mathbf{w}_{T+1}) \right\|  \\&+  \frac{\mu}{n} + \frac{8D_{\ast}\log(4/\delta)}{n} + 4\sqrt{\frac{2 \mathbb{E} [ \| \nabla f(\mathbf{w}^{\ast};z,z') \|^2 ] \log(4/\delta)}{n}}.
\end{align}
Therefore, plugging (\ref{ol6k7jupo67uj}), (\ref{gdpl}) and (\ref{oprjg09rjj}) into (\ref{jfhuiohfiooij}), we derive with probability at least $1-\delta$
\begin{align}\label{feikusjij}\nonumber
&\hphantom{{}={}}F(\mathbf{w}_{T+1}) - F(\mathbf{w}^{\ast}) \leq \mathcal{O} \Big((1-\frac{\mu}{\beta})^{T} \Big)  \\&+  \mathcal{O} \Big(\frac{\log^2(1/\delta)}{n^2} + \frac{ \mathbb{E} [ \| \nabla f(\mathbf{w}^{\ast};z,z') \|^2 ] \log(1/\delta)}{n}\Big) .
\end{align}
When $f$ is nonegative and $\beta$-smooth, from Lemma 4.1 of \cite{srebro2010optimistic}, we have 
\begin{align*}
\| \nabla f(\mathbf{w}^{\ast};z,z') \|^2 \leq 4 \beta f(\mathbf{w}^{\ast};z,z'),
\end{align*}
thus we have
\begin{align}\label{nonioggehe}
\mathbb{E} [ \| \nabla f(\mathbf{w}^{\ast};z,z') \|^2 ] \leq 4 \beta \mathbb{E}  f(\mathbf{w}^{\ast};z,z') = 4 \beta  F(\mathbf{w}^{\ast}).
\end{align}
By (\ref{nonioggehe}), (\ref{feikusjij}) implies 
\begin{align*}
&\hphantom{{}={}}F(\mathbf{w}_{T+1}) - F(\mathbf{w}^{\ast}) \leq \mathcal{O} \Big((1-\frac{\mu}{\beta})^{T} \Big)  \\&+  \mathcal{O} \Big(\frac{\log^2(1/\delta)}{n^2} + \frac{ F(\mathbf{w}^{\ast}) \log(1/\delta)}{n}\Big) .
\end{align*}
The proof is complete.
\end{proof}
\subsection{Proof of Theorem \ref{theorem8}}
We first introduce some necessary lemmas on the empirical risk.
\begin{lemma}\cite{li2021improved}\label{lemma10}
Let $\{ \mathbf{w}_t\}_t$ be the sequence produced by Algorithm \ref{alg:example} with $\eta_t \leq \frac{1}{2\beta}$ for all $t \in \mathbb{N}$. Suppose Assumptions \ref{assum4}  and \ref{assum5} hold. Then, for any $\delta \in (0,1)$, with probability at least $1-\delta$, there holds that
\begin{align*}
\sum_{k=1}^t \eta_k \|\nabla F_S(\mathbf{w}_k)  \|^2 = \mathcal{O} \Big(\log \frac{1}{\delta}  +  \sum_{k=1}^t \eta_k^{2} \Big).
\end{align*}
\end{lemma}
\begin{lemma}\cite{li2021improved}\label{lemma11}
Let $\{ \mathbf{w}_t\}_t$ be the sequence produced by Algorithm \ref{alg:example} with $\eta_t \leq \frac{1}{2\beta}$ for all $t \in \mathbb{N}$. Suppose Assumptions \ref{assum4}  and \ref{assum5} hold. Then, for any $\delta \in (0,1)$, with probability at least $1-\delta$, there holds uniformly for all $t = 1,..,T$
\begin{align*}
&\| \mathbf{w}_{t+1} - \mathbf{w}^{\ast}\| \\
= &\mathcal{O} \Big( \Big( \sum_{k=1}^T \eta_k^2 \Big)^{1/2} + 1 \Big) \Big( \Big( \sum_{k=1}^t \eta_k \Big)^{1/2} + 1 \Big) \log \Big(\frac{1}{\delta} \Big).
\end{align*}
\end{lemma}
\begin{lemma}\cite{li2021improved}\label{lemma12}
Let $\{ \mathbf{w}_t\}_t$ be the sequence produced by Algorithm \ref{alg:example} with $\eta_t = \frac{2}{\mu (t+t_0)}$ such that $t_0 \geq \max \left\{\frac{4\beta}{\mu},1 \right\}$ for all $t \in \mathbb{N}$. Suppose Assumptions \ref{assum4}  and \ref{assum5} hold, and suppose $F_S$ satisfies Assumption \ref{assum3} with parameter $2\mu$.
Then, for any $\delta > 0$, with probability at least $1 - \delta$, there holds that
\begin{align*}
F_S(\mathbf{w}_{T+1}) - F_S(\hat{\mathbf{w}}^{\ast}) =
             \mathcal{O}\Big (\frac{\log(T) \log^3(1/\delta)}{T} \Big) .
\end{align*}
\end{lemma}
\begin{lemma}\cite{lei2021learning}\label{lemma13}
Let $e$ be the base of the natural logarithm.
There holds the following elementary inequalities.

\noindent (a) If $\theta \in (0,1)$, then $\sum_{k =1}^t k^{-\theta} \leq t^{1-\theta}/(1-\theta)$;

\noindent (b) If $\theta = 1$, then $\sum_{k =1}^t k^{-\theta} \leq \log (et)$;

\noindent (c) If $\theta > 1$, then $\sum_{k =1}^t k^{-\theta} \leq \frac{\theta}{\theta - 1}$.
\end{lemma}
Now, we begin to prove Theorem \ref{theorem8}.
\begin{proof}
Similar to the proof of Theorem \ref{theorem6}. Firstly, we have
\begin{align*}
&\sum_{t = 1}^T \eta_t \| \nabla F(\mathbf{w}_t) \|^2 \\
\leq & 2\sum_{t = 1}^T \eta_t \| \nabla F(\mathbf{w}_t) - \nabla F_S(\mathbf{w}_t) \|^2 + 2\sum_{t = 1}^T \eta_t\| \nabla F_S(\mathbf{w}_t)\|^2 \\
\leq &2\sum_{t = 1}^T  \eta_t \max_{t=1,...,T}\| \nabla F(\mathbf{w}_t) - \nabla F_S(\mathbf{w}_t) \|^2 \\&+ \mathcal{O} \Big(\sum_{t=1}^T \eta_t^2 + \log \Big(\frac{1}{\delta} \Big) \Big)
\end{align*}
with probability at least $1 - \delta/3$,
which also implies that with probability at least $1 - 2\delta/3$,
\begin{align}\label{ru4yt9785gui}\nonumber
&\Big(\sum_{t = 1}^T \eta_t \Big)^{-1} \sum_{t = 1}^T \eta_t \| \nabla F(\mathbf{w}_t) \|^2 \\\nonumber
\leq &2  \max_{t=1,...,T}\| \nabla F(\mathbf{w}_t) - \nabla F_S(\mathbf{w}_t) \|^2 \\\nonumber&+ \Big(\sum_{t = 1}^T \eta_t \Big)^{-1}\mathcal{O} \Big(\sum_{t=1}^T \eta_t^2 + \log \Big(\frac{1}{\delta} \Big) \Big)\\\nonumber
\leq &\Big(\sum_{t = 1}^T \eta_t \Big)^{-1}\mathcal{O} \Big(\sum_{t=1}^T \eta_t^2 + \log \Big(\frac{1}{\delta} \Big) \Big) \\\nonumber
&+ 2 \max_{t=1,...,T} \Big[C\beta \max \Big\{\|\mathbf{w}_t - \mathbf{w}^{\ast}\|  ,\frac{1}{n} \Big\}  \\
&\times ( \sqrt{\frac{d+\log \frac{12 \log_2(\sqrt{2}R_1 n + 1)}{\delta}}{n}} + \frac{d+\log \frac{12 \log_2(\sqrt{2}R_1 n + 1)}{\delta}}{n}  ) \Big]^2.
\end{align}
According to Lemma \ref{lemma11} and Lemma \ref{lemma13}, with probability $1-\delta/3$, we have the following inequality uniformly for all $t = 1,..,T$ 
\begin{align}\label{rihg984uyg8}
\|\mathbf{w}_{t+1} - \mathbf{w}^{\ast}\|=  &
\begin{cases}
             \mathcal{O}(\log(1/\delta)) T^{\frac{2-3\theta}{2}}, & \text{ if   } \theta < 1/2   \\ 
             \mathcal{O}(\log(1/\delta)) T^{\frac{1}{4}} \log^{1/2}T, &\text{ if   }\theta = 1/2\\ 
             \mathcal{O}(\log(1/\delta)) T^{\frac{1 -\theta}{2}},  & \text{ if   }\theta > 1/2.
\end{cases}
\end{align}
Moreover, according to Lemma \ref{lemma13}, we have
\begin{align}\label{ineq987}\nonumber
&\Big(\sum_{t = 1}^T \eta_t \Big)^{-1}\mathcal{O} \Big(\sum_{t=1}^T \eta_t^2 + \log \Big(\frac{1}{\delta} \Big) \Big) \\=  &
\begin{cases}
                \mathcal{O}(\log(1/\delta) T^{-\theta}), &\quad \text{if   } \theta < 1/2   \\  \mathcal{O}(\log(T/\delta)T^{-\frac{1}{2}}) , &\quad \text{if   }\theta = 1/2\\ 
\mathcal{O}(\log(1/\delta) T^{\theta - 1}),  &\quad \text{if   }\theta > 1/2.
\end{cases}
\end{align}
Denote $\xi_{n,d,\delta} = \frac{d + \log \frac{ \log n }{\delta}}{n}\log^2(1/\delta) $.
Plugging (\ref{rihg984uyg8}) and (\ref{ineq987}) into (\ref{ru4yt9785gui}), we finally get that with probability $1-\delta$
\begin{align*}
&\hphantom{{}={}}\Big(\sum_{t = 1}^T \eta_t \Big)^{-1} \sum_{t = 1}^T \eta_t \| \nabla F(\mathbf{w}_t) \|^2 \\ \nonumber
&=
\begin{cases}
             \mathcal{O}(\xi_{n,d,\delta}) T^{2-3\theta} + \mathcal{O}(\log(1/\delta) T^{-\theta}), &\quad \text{if   } \theta < 1/2   \\ 
             \mathcal{O}(\xi_{n,d,\delta}) T^{\frac{1}{2}} \log T + \mathcal{O}(\log(T/\delta)T^{-\frac{1}{2}}) , &\quad \text{if   }\theta = 1/2\\ 
             \mathcal{O}(\xi_{n,d,\delta}) T^{1-\theta} + \mathcal{O}(\log(1/\delta) T^{\theta - 1}),  &\quad \text{if   }\theta > 1/2,
\end{cases}
\end{align*}
If $\theta < 1/2 $, we choose $T \asymp (nd^{-1})^{\frac{1}{2(1-\theta)}}$. If  $\theta = 1/2 $, we set $T \asymp nd^{-1}$. While if $\theta > 1/2 $, we set $T \asymp (nd^{-1})^{\frac{1}{2(1-\theta)}}$.
Then we can prove the learning rates of Theorem \ref{theorem8}.
The proof is complete.
\end{proof}
\subsection{Proof of Theorem \ref{theorem9}}
\begin{proof}
If Assumptions \ref{assum2} and \ref{assum4} hold and $F_S$ satisfies Assumption \ref{assum3} with parameter $\mu$, from (\ref{iqe23}), we know that
for all $\mathbf{w} \in \mathcal{W}$ and any $\delta >0 $, 
with probability at least $1- \delta/2$, when $n \geq \frac{c\beta^2(d+ \log \frac{16 \log_2(\sqrt{2}R_1 n + 1)}{\delta})}{\mu^2}$, we have
\begin{align}\label{theo1111}\nonumber
&\hphantom{{}={}}\Big(\sum_{t = 1}^T \eta_t \Big)^{-1} \sum_{t = 1}^T \eta_t \| \nabla F(\mathbf{w}_t) \|^2 \\\nonumber
&\leq 16\Big(\sum_{t = 1}^T \eta_t \Big)^{-1} \sum_{t = 1}^T \eta_t\left\| \nabla F_S(\mathbf{w}_t) \right\|^2 +  \frac{4\mu}{n^2} \\&+ \frac{32D_{\ast}^2\log^2(8/\delta)}{n^2} + \frac{16 \mathbb{E} [ \| \nabla f(\mathbf{w}^{\ast};z,z') \|^2 ] \log(8/\delta)}{n}.
\end{align}
When $\eta_t = \eta_1 t^{- \theta} , \theta \in (0,1)$ with $\eta_1 \leq \frac{1}{2\beta}$ and Assumptions \ref{assum4} and \ref{assum5} hold, according to Lemma \ref{lemma10} and (\ref{ineq987}), we obtain the following inequality with probability at least $1-\delta/2$,
\begin{align}\label{theo1112}\nonumber
&\Big(\sum_{t = 1}^T \eta_t \Big)^{-1} \sum_{t = 1}^T \eta_t\left\| \nabla F_S(\mathbf{w}_t) \right\|^2 \\=  &
\begin{cases}
                \mathcal{O}(\log(1/\delta) T^{-\theta}), &\quad \text{if   } \theta < 1/2   \\  \mathcal{O}(\log(T/\delta)T^{-\frac{1}{2}}) , &\quad \text{if   }\theta = 1/2\\ 
\mathcal{O}(\log(1/\delta) T^{\theta - 1}),  &\quad \text{if   }\theta > 1/2.
\end{cases}
\end{align}
Plugging  (\ref{theo1112}) into (\ref{theo1111}), with probability at least $1-\delta$, we derive that
\begin{align*}
&\hphantom{{}={}}\Big(\sum_{t = 1}^T \eta_t \Big)^{-1} \sum_{t = 1}^T \eta_t \| \nabla F(\mathbf{w}_t) \|^2 \\ \nonumber
&=
\begin{cases}
             \mathcal{O}(\xi_{n,\mathbf{w}^{\ast},\delta})   + \mathcal{O}(\log(1/\delta) T^{-\theta}), &\quad \text{if   } \theta < 1/2,   \\ 
             \mathcal{O}(\xi_{n,\mathbf{w}^{\ast},\delta})  + \mathcal{O}(\log(T/\delta) T^{-\frac{1}{2}}), &\quad \text{if   }\theta = 1/2,\\ 
             \mathcal{O}(\xi_{n,\mathbf{w}^{\ast},\delta}) + \mathcal{O}(\log(1/\delta) T^{\theta - 1}),  &\quad \text{if   }\theta > 1/2,
\end{cases}
\end{align*}
where $\mathcal{O}(\xi_{n,\mathbf{w}^{\ast},\delta}) = \mathcal{O} \Big( \frac{\log^2(1/\delta)}{n^2} + \frac{ F(\mathbf{w}^{\ast}) \log(1/\delta)}{n} \Big)$, and where $F(\mathbf{w}^{\ast})$ exists due to (\ref{nonioggehe}).

When $\theta < 1/2 $, we set $T \asymp n^{\frac{2}{\theta}}$, then we obtain the following result with probability at least $1-\delta$
\begin{align*}
\Big(\sum_{t = 1}^T \eta_t \Big)^{-1} \sum_{t = 1}^T \eta_t \| \nabla F(\mathbf{w}_t) \|^2  = 
\mathcal{O} \Big(\frac{\log^2(\frac{1}{\delta})}{n^2} +\frac{ F(\mathbf{w}^{\ast}) \log(\frac{1}{\delta})}{n}\Big).
\end{align*}
When $\theta = 1/2 $, we set $T \asymp n^4$, then we obtain the following result with probability at least $1-\delta$
\begin{align*}
\Big(\sum_{t = 1}^T \eta_t \Big)^{-1} \sum_{t = 1}^T \eta_t \| \nabla F(\mathbf{w}_t) \|^2  = 
\mathcal{O} \Big(\frac{\log^2(\frac{1}{\delta})}{n^2} +\frac{ F(\mathbf{w}^{\ast}) \log(\frac{1}{\delta})}{n}\Big).
\end{align*}
When $\theta > 1/2 $, we set $T \asymp n^{\frac{2}{1-\theta}}$, then we obtain the following result with probability at least $1-\delta$
\begin{align*}
\Big(\sum_{t = 1}^T \eta_t \Big)^{-1} \sum_{t = 1}^T \eta_t \| \nabla F(\mathbf{w}_t) \|^2  = 
\mathcal{O} \Big(\frac{\log^2(\frac{1}{\delta})}{n^2}+\frac{ F(\mathbf{w}^{\ast}) \log(\frac{1}{\delta})}{n} \Big).
\end{align*}
The proof is complete.
\end{proof}
\subsection{Proof of Theorem \ref{theorem10}}
\begin{proof}
Since $F_S$ satisfies the PL assumption with parameter $2\mu$, we have
\begin{align}\label{ytj6u65}
F(\mathbf{w}) -F(\mathbf{w}^{\ast}) \leq \frac{\left\| \nabla F(\mathbf{w}) \right\|^2 }{4 \mu}, \quad \forall  \mathbf{w} \in \mathcal{W}.
\end{align} 
So to bound $F(\mathbf{w}_{T+1}) - F(\mathbf{w}^{\ast}) $, we need to bound the term $\left\| \nabla F(\mathbf{w}_{T+1}) \right\|^2$.
And there holds that
\begin{align}\label{ciejf0i4oj}\nonumber
\left\| \nabla F(\mathbf{w}_{T+1}) \right\|^2 
&\leq 2 \left\| \nabla F(\mathbf{w}_{T+1})- \nabla F_S(\mathbf{w}_{T+1}) \right\|^2 \\
&+ 2 \| \nabla F_S(\mathbf{w}_{T+1}) \|^2.
\end{align}

From Theorem \ref{theorem3},
if Assumptions \ref{assum2} and \ref{assum4} hold and $F_S$ satisfies Assumption \ref{assum3},
for all $\mathbf{w} \in \mathcal{W}$ and any $\delta >0 $,
  with probability at least $1- \delta/2$, when $n \geq \frac{c\beta^2(d+ \log \frac{16 \log_2(\sqrt{2}R_1 n + 1)}{\delta})}{\mu^2}$, there holds 
\begin{align}\label{jreg09ji}\nonumber
&\hphantom{{}={}}\left\| \nabla F (\mathbf{w}_{T+1} )- \nabla F_S(\mathbf{w}_{T+1}) \right\|\leq \left\| \nabla F_S(\mathbf{w}_{T+1}) \right\| +  \frac{2\mu}{n} \\&+ \frac{8D_{\ast}\log(8/\delta)}{n} + 4\sqrt{\frac{8 \beta F(\mathbf{w}^{\ast}) \log(8/\delta)}{n}}, 
\end{align}
where $F(\mathbf{w}^{\ast})$ follows from (\ref{nonioggehe}).
For the second term $\| \nabla F_S(\mathbf{w}_{T+1}) \|^2$, according to the smoothness property of $F_S$ (see (\ref{09ur90hgoij})) and
Lemma \ref{lemma12}, it can be derived that with probability at least $1- \delta/2$
\begin{align}\label{6u6u6}
\| \nabla F_S(\mathbf{w}_{T+1})  \|^2 =  \mathcal{O}\Big (\frac{\log(T) \log^3(1/\delta)}{T} \Big).
\end{align}
Plugging (\ref{6u6u6}) into (\ref{jreg09ji}), we can derive that 
\begin{align}\label{5yh45h}\nonumber
&\left\| \nabla F(\mathbf{w}_{T+1})- \nabla F_S(\mathbf{w}_{T+1}) \right\|^2 \\
= &\mathcal{O} \Big(\frac{\log T \log^3(1/\delta)}{T} \Big) + \mathcal{O} \Big(\frac{\log^2(1/\delta)}{n^2} + \frac{ F(\mathbf{w}^{\ast}) \log(1/\delta)}{n}\Big).
\end{align}
Therefore, substituting (\ref{5yh45h}) and (\ref{6u6u6}) into (\ref{ciejf0i4oj}), we derive that
\begin{align}\label{lsj9rhbtrht5zj}\nonumber
&\left\| \nabla F(\mathbf{w}_{T+1}) \right\|^2 \\
= &\mathcal{O} \Big(\frac{\log T \log^3(1/\delta)}{T} \Big) + \mathcal{O} \Big(\frac{\log^2(1/\delta)}{n^2} + \frac{ F(\mathbf{w}^{\ast}) \log(1/\delta)}{n}\Big).
\end{align}
Further substituting (\ref{lsj9rhbtrht5zj}) into (\ref{ytj6u65})
and choosing $T \asymp n^2$,
we finally obtain with probability at least $1- \delta$
\begin{align*}
F(\mathbf{w}_{T+1}) -F(\mathbf{w}^{\ast}) = \mathcal{O} \Big(\frac{\log n \log^3(\frac{1}{\delta})}{n^2} + \frac{ F(\mathbf{w}^{\ast}) \log(\frac{1}{\delta})}{n}\Big).
\end{align*}
The proof is complete.
\end{proof}

\section{Conclusion}\label{Section5}
We studied the generalization performance of nonconvex pairwise learning given that it was rarely studied. 
We established several uniform convergences of gradients, based on which we provided a series of learning rates for ERM, GD, and SGD. We first investigated the general nonconvex setting and then the nonconvex learning with a gradient dominance curvature condition. Former demonstrated how the optimal iterative numbers should be selected to balance the generalization and optimization, shed insights on the role of early-stopping,
and the latter highlight the established learning rates which are significantly faster than the state-of-the-art, even up to $\mathcal{O}(1/n^2)$. Overall, we provide a relatively systematic study of nonconvex pairwise learning. 


%

\appendices
\section{Comparison with the related work}\label{appendix2}
In Table \ref{tab1}, we compare our learning rates with the most related work \cite{lei2018generalization,lei2020sharper,lei2021generalization}. They also study the pairwise learning framework \cite{lei2018generalization,lei2020sharper,lei2021generalization}, as discussed in the main paper. For brevity, we just give the main assumptions on the objectives. Other assumptions of the related work and this paper, such as the boundedness of $f$, the choice of step sizes, and the Bernstein condition, etc, are omitted. Since \cite{lei2018generalization,lei2020sharper,lei2021generalization} all assume that $f$ is nonnegative, we thus also omit it. Here, ``Lip'' means Lipschitz continuous. A function $f$ is $G$-Lipschitz continuous if $|f(\mathbf{w}) - f(\mathbf{w}')| \leq G \| \mathbf{w} - \mathbf{w}'  \|$ for all $\mathbf{w}, \mathbf{w}' \in \mathcal{W}$. If $f$ is continuously differentiable, then $f$ is $G$-Lipschitz if and only if $\|\nabla f(\mathbf{w})\| \leq G$ for all $\mathbf{w} \in \mathbb{R}^{d}$. $\|\nabla f(\mathbf{w})\| \leq G$ is the boundedness assumption of gradients mentioned in Remark \ref{r08gur8ghiothjb}. 
A function $f$ is $\xi$-strongly-convex w.r.t. $\|\cdot \|$ if $f(\mathbf{w}) - f(\mathbf{w}')-\langle \mathbf{w} - \mathbf{w}', \nabla f (\mathbf{w}') \rangle \geq \frac{\xi}{2} \| \mathbf{w} - \mathbf{w}' \|^2$ for all $\mathbf{w}, \mathbf{w}' \in \mathcal{W}$. When $\xi=0$, we say $f$ is convex. ``Low noise'' means that $\mathbb{E} [ \| \nabla f(\mathbf{w}^{\ast};z,z') \|^2 ]= \mathcal{O} \left(\frac{1}{n} \right)$ or $F(\mathbf{w}^{\ast}) = \mathcal{O} \Big(\frac{1}{n} \Big)$. ``Assum 5'' means Assumption \ref{assum5}.  ``PL'' means Assumption \ref{assum3}. ``Sub-exponential'' means Assumption \ref{assum1}. ``Generalization Performance Gap'' means $F(\mathbf{w}(S)) - F(\mathbf{w}^{\ast})$, also referred to as excess risk. ``Norm of Gradient'' means $\|  \nabla F(\mathbf{w}(S)) \|$. ``In expectation'' means studying the learning rate of $\mathbb{E}[F(\mathbf{w}(S))] - F(\mathbf{w}^{\ast})$. ``In probability'' learning rates are beneficial to understand the robustness of optimization algorithms and is much more challenging to be derived than ``In expectation'' ones \cite{bousquet2002stability,bousquet2020sharper,klochkov2021stability}.

From Table \ref{tab1}, one can see that we have provided a relatively systematic study of nonconvex pairwise learning. In the general nonconvex setting, when considering the low-dimensional case, our $\mathcal{O}\Big(\sqrt{\frac{d}{n}}\Big)$ learning rate of SGD is comparable to the corresponding rate $\mathcal{O}\Big(\frac{\log n }{\sqrt{n}}\Big)$ established in convex learning in \cite{lei2020sharper}. In nonconvex learning with the PL condition, we have established significantly faster learning rates than the state-of-the-art, even up to $\mathcal{O}(1/n^2)$, which is the first $\mathcal{O}(1/n^2)$-type of rate for pairwise learning. Note that although the ``Low noise'' condition is also used in \cite{lei2021generalization}, \cite{lei2021generalization} only presents $\mathcal{O}\Big(\frac{\log n }{\sqrt{n}}\Big)$ order rate, moreover, derived just in expectation.

\begin{table*}[!htbp]
    \centering
    \begin{tabular}{c|c|c|c|c|c}
    \hline
Reference  &  Algorithm  &   Assumptions on The Objectives & Generalization Performance & Learning Rate & Types\\\hline
\multirow{1}{*}{\cite{lei2018generalization}} &\multirow{1}{*}{RRM}& Lip,  Strongly-convex  &  Generalization Performance Gap  & $\mathcal{O}\Big(\frac{1}{n}\Big)$& In probability\\\hline
\multirow{3}{*}{\cite{lei2020sharper}}  &\multirow{2}{*}{RRM}&  Lip, Strongly-convex &  Generalization Performance Gap&$\mathcal{O}\Big(\frac{\log n }{\sqrt{n}}\Big)$&In probability\\\cline{3-6}
& & Smooth,  Strongly-convex  &  Generalization Performance Gap  & $\mathcal{O}\Big(\frac{1}{n}\Big)$& In expectation\\\cline{2-2}\cline{3-6}
&\multirow{1}{*}{SGD}&  Smooth, Convex & Generalization Performance Gap & $\mathcal{O}\Big(\frac{\log n }{\sqrt{n}}\Big)$&In probability\\ \hline
\multirow{6}{*}{\cite{lei2021generalization}}   & \multirow{6}{*}{SGD} & Smooth, Convex, Low noise & Generalization Performance Gap & $\mathcal{O}\Big(\frac{1}{n}\Big)$&In expectation\\\cline{3-6}
&  & Lip, Convex & Generalization Performance Gap & $\mathcal{O}\Big(\frac{1}{\sqrt{n}}\Big)$&In expectation\\\cline{3-6}
& & Smooth, Strongly-convex & Generalization Performance Gap & $\mathcal{O}\Big(\frac{1}{n}\Big)$&In expectation\\ \cline{3-6}
& & Lip, Strongly-convex & Generalization Performance Gap & $\mathcal{O}\Big(\frac{1}{n}\Big)$&In expectation\\ \cline{3-6}
&  & Smooth, Assum \ref{assum5}  & Norm of Gradient  & $\mathcal{O}\Big(\sqrt{\frac{d}{n}}\Big)$&In probability\\\cline{3-6}
&  & Lip, Smooth, PL & Generalization Performance Gap  & $\mathcal{O}\Big(\frac{1}{n^{2/3}}\Big)$ &In expectation\\ \hline
\multirow{6}{*}{This work} &\multirow{2}{*}{ERM}
&  Sub-exponential &  Norm of Gradient     &   $\mathcal{O}\Big(\sqrt{\frac{d}{n}}\Big)$ &In probability \\ \cline{3-6}
& & Sub-exponential, PL, Low noise &  Generalization Performance Gap  & $\mathcal{O}\Big(\frac{1}{n^2}\Big)$& In probability \\ \cline{2-2}\cline{3-6}
&\multirow{2}{*}{GD}&  Smooth &  Norm of Gradient     &   $\mathcal{O}\Big(\sqrt{\frac{d}{n}}\Big)$ &In probability \\  \cline{3-6}
& & Smooth, PL, Low noise  &  Generalization Performance Gap  & $\mathcal{O}\Big(\frac{1}{n^2}\Big)$& In probability\\ \cline{2-2} \cline{3-6}
& \multirow{2}{*}{SGD} & Smooth, Assum \ref{assum5}  & Norm of Gradient & $\mathcal{O}\Big(\sqrt{\frac{d}{n}}\Big)$&In probability\\\cline{3-6}
& & Smooth, PL, Assum \ref{assum5}, Low noise & Generalization Performance Gap & $\mathcal{O}\Big(\frac{1}{n^2}\Big)$&In probability\\ \hline 
    \end{tabular}
    \vspace*{+1.5mm}
    \caption{
    Summary of Results.
      \label{tab1}}
\end{table*}
\section{Preliminaries of Section \ref{proof4.2}}\label{appendix1}
We first provide a lemma on uniform localized convergence.
\begin{lemma}[\cite{xu2020towards,xu2020toward}]\label{lemma1}
For a function class $\mathcal{G} = \{ g_f: f \in \mathcal{F} \}$ and functional $T: \mathcal{F} \mapsto [0,R]$, suppose there is a function $\psi(r;\delta)$ (possibly depending on the samples), which is non-decreasing with respect to $r$ and satisfies that $\forall \delta \in (0,1)$, with probability $1-\delta$, $\forall r \in [0,R]$,
\begin{align*}
\sup_{f \in \mathcal{F}: T(f) \leq r } (P - P_n) g_f \leq \psi(r;\delta).
\end{align*}
Then, given any $\delta \in (0,1)$ and $r_0 \in (0,R]$, with probability at least $1-\delta$, for all $f \in \mathcal{F}$,
\begin{align*}
(P - P_n) g_f  \leq \psi \Big(\max \{2T(f), r_0\}; \frac{\delta}{2\log_2\frac{2R}{r_0}} \Big).
\end{align*}
\end{lemma}
In the following, we introduce some important definitions and lemmas on generic chaining. 
\begin{definition}[Orlicz-$\alpha$ Norm \cite{dirksen2015tail}]\label{defi1}
For every $\alpha > 0$, we define the $Orlicz-\alpha$ norm of a random $v$:
\begin{align*}
\| v \|_{Orlicz-\alpha} = \inf \{ K>0 : \mathbb{E} \exp((|v|/K)^{\alpha}) \leq 2\}.
\end{align*}
A random variable (or vector) $X \in \mathbb{R}^d$ is $K$-sub-exponential if $\forall \lambda \in \mathbb{R}^d$, we have
\begin{align*}
\| \lambda^T X \|_{Orlicz-1} \leq K \| \lambda \|_2.
\end{align*}
A random variable (or vector) $X \in \mathbb{R}^d$ is $K$-sub-Gaussian if $\forall \lambda \in \mathbb{R}^d$, we have
\begin{align*}
\| \lambda^T X \|_{Orlicz-2} \leq K \| \lambda \|_2.
\end{align*}
\end{definition}
\begin{definition}[Orlicz-$\alpha$ Processes \cite{dirksen2015tail}]
Let $\{ X_f \}_{f \in \mathcal{F}}$ be a sequence of random variables. We call $\{ X_f \}_{f \in \mathcal{F}}$ be an Orlicz-$\alpha$ process for a metric $d(\cdot,\cdot)$ on $\mathcal{F}$ if
\begin{align*}
\|  X_{f_1} - X_{f_2} \|_{Orlicz-\alpha} \leq d(f_1,f_2), \forall f_1, f_2 \in \mathcal{F}.
\end{align*}
Typically, we call Orlicz-$1$ process the sub-exponential increments and Orlicz-$2$ process the sub-Gaussian increments.
\end{definition}
\begin{definition}[mixed sub-Gaussian-sub-exponential increments \cite{dirksen2015tail}]\label{defi3}
A process $(X_{\mathbf{w}})_{\mathbf{w} \in \mathcal{W}}$ is called mixed mixed sub-Gaussian-sub-exponential increments w.r.t. the metric pair $(d_1(\cdot,\cdot), d_2(\cdot,\cdot))$ if for all $\mathbf{w}_1,\mathbf{w}_2 \in \mathcal{W}$ and $\forall u \geq 0$,
\begin{align*}
Pr \left(\| X_{\mathbf{w}_1} - X_{\mathbf{w}_2} \| \geq \sqrt{u} d_2(\mathbf{w}_1,\mathbf{w}_2)+ u d_1(\mathbf{w}_1,\mathbf{w}_2) \right) \leq 2e^{-u},
\end{align*}
where "Pr" means probability.
\end{definition}
\begin{definition}[$\gamma_{\alpha}$-functional \cite{dirksen2015tail}]
A sequence $F = (\mathcal{F}_n)_{n \geq 0}$ of subsets of $\mathcal{F}$ is called admissible if  $|\mathcal{F}_0|=1$ and $|\mathcal{F}_n|=2^{2^n}$ for all $n \geq 1$. For any $0<\alpha<\infty$, the $\gamma_{\alpha}$-functional of $(\mathcal{F},d)$ is defined as
\begin{align*}
\gamma_{\alpha}(\mathcal{F},d) = \inf_{F} \sup_{f \mathcal{F}} \sum_{n=0}^{\infty} 2^{n/\alpha} d(f,\mathcal{F}_n),
\end{align*}
where the infimum is taken over all admissible sequences and we define $d(f,\mathcal{F}_n) = \inf_{t\in \mathcal{F}_n}d(f,t)$.
\end{definition}
\begin{definition}[Covering number \cite{lei2021generalization,wainwright2019high}]
Assume $(\mathcal{H},d)$ be a metric space and $\mathcal{F} \subseteq \mathcal{H}$. For any $\epsilon >0$, a set $\mathcal{F}_c$ is called an $\epsilon$-cover of $\mathcal{F}$ if for any $f \in \mathcal{F}$ we have an element $g \in \mathcal{F}_c$ such that $d(f,g) \leq \epsilon$. We denote $N(\mathcal{F},d,\epsilon) $ the covering number as the cardinality of the minimal $\epsilon$-cover of $\mathcal{F}$:
\begin{align*}
N(\mathcal{F},d,\epsilon) = \min \{ |\mathcal{F}_c| : \mathcal{F}_c \text{ is an $\epsilon$ cover of $\mathcal{F}$ }\}.
\end{align*}
\end{definition}
\begin{lemma}[Dudley's integral bound for $\gamma_{\alpha}$-functional \cite{dirksen2015tail,xu2020towards}]\label{lemma1.5}
There exists a constant $c_{\alpha}$ dependent to $\alpha$ such that 
\begin{align*}
\gamma_{\alpha}(\mathcal{F},d)  \leq c_{\alpha} \int_{0}^{\infty} \left( \log N(\mathcal{F},d,\epsilon) \right)^{1/\alpha} d\epsilon,
\end{align*}
where $N(\mathcal{F},d,\epsilon)$ is the covering number of $\mathcal{F}$.
\end{lemma}
\begin{lemma}[generic chaining for a process with mixed sub-Gaussian-sub-exponential increments \cite{dirksen2015tail}]\label{lemma2}
If $(X_f)_{f \in \mathcal{F}}$ has mixed sub-Gaussian-sub-exponential increments w.r.t. the metric pair $(d_1(\cdot,\cdot), d_2(\cdot,\cdot))$, then there are absolute constants $c_1$, $c_2$ such that $\forall \delta \in (0,1)$, with probability at least $1-\delta$,
\begin{multline*}
\sup_{\mathbf{w} \in \mathcal{W}} \| X_f - X_{f_0}  \| \leq c_1 (\gamma_{2}(\mathcal{F},d_2) + \gamma_{1}(\mathcal{F},d_1) ) +\\
c_2 \Big ( \sqrt{\log \frac{1}{\delta}} \sup_{f_1,f_2 \in \mathcal{F}} [d_2(f_1,f_2)] + \log \frac{1}{\delta} \sup_{f_1,f_2 \in \mathcal{F}} [d_1(f_1,f_2)] \Big).
\end{multline*}
\end{lemma}
After the introduction of generic chaining, we are to provide two general Bernstein inequalities for pairwise learning. Existing Bernstein inequalities in pairwise learning are almost all provided for bounded random variables \cite{clemencon2008ranking,hoeffding1963probability,peel2010empirical,pitcan2017a}. 
Moreover, although the Bernstein inequality of pairwise learning is often used in machine learning literature, it isn't easy to find clear  proof. Here, we derive two more general forms of the Bernstein inequality of pairwise learning, which can be satisfied by various unbounded variables, and provide the proof, which may benefit the wider applicability of the Bernstein inequality.
\begin{lemma}\label{lemma3}
Suppose $T$ is a random variable that can be written as
\begin{align*}
T = \sum_{i =1}^Np_i T_i, \quad where \sum p_i =1, p_i \geq 0 \quad \forall i.
\end{align*}
Then we have
\begin{align*}
Pr(T \geq t) \leq \sum_{i  =1 }^N p_i \mathbb{E}[e^{\lambda (T_i - t)}].
\end{align*}
\end{lemma}
\begin{proof}
According to Markov's inequality
\begin{align*}
Pr(T \geq t) \leq e^{-\lambda t} \mathbb{E}[e^{\lambda T}] = \mathbb{E} [e^{\lambda(T-t)}].
\end{align*}
Then, by Jensen's inequality,
\begin{align*}
\exp(\lambda T) = \exp(\lambda \sum p_iT_i) \leq \sum p_i \exp(\lambda T_i).
\end{align*}
Thus, we have
\begin{align*}
Pr(T\geq t) \leq \sum p_i \mathbb{E} [e^{\lambda(T_i - t)}].
\end{align*}
The proof is complete.
\end{proof}
\begin{lemma}\cite{wainwright2019high}\label{lemma4}
Given a random variable $X$ with mean
$\mu = \mathbb{E}[X]$ and variance $\sigma^2 = \mathbb{E}[X^2] - \mu^2$, we say that Bernstein’s condition with parameter $b$
holds if
\begin{align*}
| \mathbb{E} [(X - \mu)^k ]   | \leq \frac{1}{2} k! \sigma^2 b^{k - 2} \quad \text{for } k = 2,3,4,....
\end{align*}
If the Bernstein's condition is statisfied, we have for all $\lambda \in (0,1/b)$
\begin{align*}
\mathbb{E} [e^{\lambda (X - \mu)}] \leq e^{\frac{\lambda^2 \sigma^2}{2(1-b\lambda)}}. 
\end{align*}
\end{lemma}
\begin{remark}\rm{}
The Bernstein condition is milder than the bounded assumption of random variables and is also satisfied by various unbounded variables \cite{wainwright2019high}.
\end{remark}
\begin{lemma}[Bernstein inequality for pairwise learning]\label{lemma5}
Let $Z_1,...,Z_n$ be real-valued and independent variables taking values in $\mathcal{Z}$, and let
$q: \mathcal{Z} \times \mathcal{Z} \mapsto \mathbb{R}$. Assume that $\sigma^2$ is the variance of $q(Z,Z')$.
Suppose that the Bernstein's condition in Lemma \ref{lemma4} holds for any $q(Z,Z')$, i.e., for $k = 2,3,4$,...
\begin{align*}
| \mathbb{E} [(q(Z,Z') - \mathbb{E}_{Z,Z'}[q(Z,Z')])^k ]   | \leq \frac{1}{2} k! \sigma^2 b^{k - 2}.
\end{align*}
Then, for any $u \geq 0$,
\begin{align*}
Pr \Big(\frac{1}{n(n-1)} \sum_{i,j \in [n], i \neq j} q(Z_i,Z_j) - \mathbb{E}_{Z,Z'}[q(Z,Z')] \\
\geq \sqrt{\frac{2\sigma^2u}{\lfloor \frac{n}{2} \rfloor}}+ \frac{bu}{\lfloor \frac{n}{2} \rfloor} \Big) \leq e^{-u},
\end{align*}
where $\lfloor \frac{n}{2} \rfloor $ is the
largest integer no greater than $\frac{n}{2}$.
\end{lemma}
\begin{proof}
According to \cite{clemencon2008ranking}, we have the following representation of $U$-statistic of order $2$
\begin{align*}
\frac{\sum_{i,j \in [n], i \neq j}}{n(n-1)}  q(Z_i,Z_j) = \frac{1}{n!} \sum_{\pi}  \frac{1}{\lfloor \frac{n}{2} \rfloor } \sum_{ i =1}^{\lfloor \frac{n}{2} \rfloor} q(Z_{\pi(i)},Z_{\pi(i+\lfloor \frac{n}{2} \rfloor)}) ,
\end{align*}
where the sum is taken over all permutations $\pi$ of $\{ 1,...,n \}$. 
Denote $V(S) = \frac{1}{\lfloor \frac{n}{2} \rfloor } \sum_{ i =1}^{\lfloor \frac{n}{2} \rfloor} q(Z_{\pi(i)},Z_{\pi(i+\lfloor \frac{n}{2} \rfloor)})$, one can see that 
\begin{align*}
\frac{1}{n(n-1)} \sum_{i,j \in [n], i \neq j} q(Z_i,Z_j)  = \frac{1}{n!} \sum_{\pi} V(S)
\end{align*}
and each term on the right-hand is a sum of $\lfloor \frac{n}{2} \rfloor$ independent random variables.

We now set $p_i=\frac{1}{n!}$, $N = n!$ and $T = \frac{1}{n(n-1)} \sum_{i,j \in [n], i \neq j} q(Z_i,Z_j) $. According to Lemma \ref{lemma3} and the fact that $\mathbb{E} [e^{\lambda V(S)}]$ is independent of the permutation of $Z_i$'s, we know that the next step is to bound the term $\mathbb{E}\exp (\lambda V(S))$. We consider that $V(S) = \frac{1}{\lfloor \frac{n}{2} \rfloor } \sum_{ i =1}^{\lfloor \frac{n}{2} \rfloor} q(Z_{i},Z_{i+\lfloor \frac{n}{2} \rfloor})$. 

Denote $\mu = \mathbb{E}_{Z,Z'}[q(Z,Z')]$. We can derive the following inequality
\begin{align*}
&\mathbb{E} [e^{\lambda (V(S) - \mu)}] = \mathbb{E} \Big[\exp\Big \{\lambda \Big ( \sum_{ i =1}^{\lfloor \frac{n}{2} \rfloor} \Big( \frac{q(Z_{i},Z_{i+\lfloor \frac{n}{2} \rfloor})}{\lfloor \frac{n}{2} \rfloor }  -  \frac{\mu}{\lfloor \frac{n}{2} \rfloor }  \Big) \Big)\Big \}\Big]\\
&\leq \Big[ \exp\Big\{\frac{\lambda^2 \sigma^2}{2\Big(\lfloor \frac{n}{2} \rfloor \Big)^2 \Big(1-\frac{b \lambda}{\lfloor \frac{n}{2} \rfloor} \Big)}\Big\} \Big]^{\lfloor \frac{n}{2} \rfloor}
=  \exp\Big\{\frac{\lambda^2 \sigma^2}{2\lfloor \frac{n}{2} \rfloor \left(1-\frac{b \lambda}{\lfloor \frac{n}{2} \rfloor} \right)}\Big\},
\end{align*}
where the inequality follows from Lemma \ref{lemma4} and the fact that for $k = 2,3,4$,...
\begin{align*}
\Big| \mathbb{E} \Big[\Big(  \frac{q(Z_{i},Z_{i+\lfloor \frac{n}{2} \rfloor})}{\lfloor \frac{n}{2} \rfloor }  -  \frac{\mu}{\lfloor \frac{n}{2} \rfloor }   \Big)^k \Big]   \Big| \leq \frac{1}{2} k! \frac{\sigma^2}{(\lfloor \frac{n}{2} \rfloor)^2} \Big(\frac{b}{\lfloor \frac{n}{2} \rfloor}\Big)^{k - 2}.
\end{align*}
Thus, we should minimize 
\begin{align*}
\frac{\lambda^2 \sigma^2}{2\lfloor \frac{n}{2} \rfloor(1-\frac{b\lambda}{\lfloor \frac{n}{2} \rfloor})} -\lambda t 
\end{align*}
to find the best $\lambda$.
We now introduce 
its Fenchel-Legendre dual function 
$\psi^{\ast}(t) = \sup_{\lambda \in (0,1/b)}\Big(\lambda t  - \frac{\lambda^2 \sigma^2}{2\lfloor \frac{n}{2} \rfloor(1-\frac{b\lambda}{\lfloor \frac{n}{2} \rfloor})} \Big)$.

Following (2.5) of \cite{boucheron2013concentration} and introducing $h(u) = 1+u-\sqrt{1+2u}$, we have 
\begin{align*}
\psi^{\ast}(t) \geq \frac{\lfloor \frac{n}{2} \rfloor\sigma^2}{b^2}  h(\frac{bt}{\sigma^2}).
\end{align*}
Therefore, 
we get
\begin{align*}
Pr\Big(\frac{1}{n(n-1)} \sum_{i,j \in [n], i \neq j} q(Z_i,Z_j) - \mathbb{E}_{Z,Z'}[q(Z,Z')] \\
\geq t \Big) \leq \exp \Big(- \frac{\lfloor \frac{n}{2} \rfloor\sigma^2}{b^2}  h(\frac{bt}{\sigma^2}) \Big).
\end{align*}
Since $h$ is an increasing function from $(0, \infty)$ onto $(0, \infty)$ with inverse function
$h^{-1}(u) = u + \sqrt{2u}$ for $u > 0$ \cite{boucheron2013concentration}, we finally get 
\begin{align*}
\psi^{\ast -1}(u) = \sqrt{\frac{2u\sigma^2}{\lfloor \frac{n}{2} \rfloor}} + \frac{bu}{\lfloor \frac{n}{2} \rfloor}.
\end{align*}
Hence, we have the following equivalent inequality
\begin{align*}
Pr \Big(\frac{1}{n(n-1)} \sum_{i,j \in [n], i \neq j} q(Z_i,Z_j) - \mathbb{E}_{Z,Z'}[q(Z,Z')] \\
\geq \sqrt{\frac{2\sigma^2u}{\lfloor \frac{n}{2} \rfloor}}+ \frac{bu}{\lfloor \frac{n}{2} \rfloor} \Big) \leq e^{-u}.
\end{align*}
The proof is complete.
\end{proof}
\begin{lemma}[Vector Bernstein inequality \cite{smale2007learning,pinelis1994optimum,xu2020towards}]\label{lemma6}
Let $Z_1,...,Z_n$ be a sequence of i.i.d. random variables taking values in a real separable Hilbert space. Assume that $\mathbb{E}[Z_i] = \mu$, $\mathbb{E}[\| Z_i - \mu \|^2] = \sigma^2$, $\forall 1 \leq i \leq n$. The vector Bernstein's condition with parameter $b$ holds if for all $1 \leq i \leq n$, 
\begin{align*}
\mathbb{E} \|Z_i - \mu\|^k  \leq \frac{k!}{2} \sigma^2 b^{k-2}, \quad k=2,3,...
\end{align*}
Then, $\forall \delta \in (0,1)$, with probability at least $1-\delta$
\begin{align*}
\Big\| \frac{1}{n} \sum_{i=1}^n Z_i - \mu\Big\| \leq \sqrt{\frac{2\sigma^2 \log \frac{2}{\delta}}{n}} + \frac{b \log \frac{2}{\delta}}{n}.
\end{align*}
\end{lemma}
\begin{lemma}[Vector Bernstein inequality for pairwise learning]\label{lemma7}
Let $Z_1,...,Z_n$ be independent random variables taking values in $\mathcal{Z}$. 
Let $q: \mathcal{Z} \times \mathcal{Z} \mapsto \mathcal{H}$, where $\mathcal{H}$ is a real separable Hilbert space.
Assume that $\sigma^2$ is the variance of $q(Z,Z')$. Suppose the  Bernstein's condition with parameter $b$ holds, i.e., for any $q(Z,Z')$, 
\begin{align*}
\mathbb{E} \|q(Z,Z) - \mathbb{E}_{Z,Z'}[q(Z,Z')]\|^k  \leq \frac{k!}{2} \sigma^2 b^{k-2}, k=2,3,...
\end{align*}
Then, $\forall \delta \in (0,1)$, with probability at least $1-\delta$
\begin{align*}
\Big\| \frac{1}{n(n-1)} \sum_{i,j \in [n], i \neq j} q(Z_i,Z_j) - \mathbb{E}_{Z,Z'}[q(Z,Z')] \Big\| \\\leq \sqrt{\frac{2\sigma^2 \log \frac{2}{\delta}}{\lfloor \frac{n}{2} \rfloor}} + \frac{b \log \frac{2}{\delta}}{\lfloor \frac{n}{2} \rfloor}.
\end{align*}
\end{lemma}
\begin{proof}
According to \cite{clemencon2008ranking}, we have the following representation of $U$-statistic of order $2$
\begin{align*}
\frac{\sum_{i,j \in [n], i \neq j}}{n(n-1)}  q(Z_i,Z_j) 
= \frac{1}{n!} \sum_{\pi} \frac{1}{\lfloor \frac{n}{2} \rfloor } \sum_{ i =1}^{\lfloor \frac{n}{2} \rfloor} q(Z_{\pi(i)}, Z_{\pi(i+\lfloor \frac{n}{2} \rfloor)}).
\end{align*}
Since $\| \cdot \|_2$ involves an expectation, it is clear that
\begin{align*}
&\Big\| \frac{1}{n(n-1)} \sum_{i,j \in [n], i \neq j} q(Z_i,Z_j) - \mathbb{E}_{Z,Z'}[q(Z,Z')] \Big\| \\
\leq & \frac{1}{n!} \sum_{\pi}\Big\| \frac{1}{\lfloor \frac{n}{2} \rfloor } \sum_{ i =1}^{\lfloor \frac{n}{2} \rfloor} q(Z_{\pi(i)}, Z_{\pi(i+\lfloor \frac{n}{2} \rfloor)}) - \mathbb{E}_{Z,Z'}[q(Z,Z')] \Big\| \\
=& \Big\| \frac{1}{\lfloor \frac{n}{2} \rfloor } \sum_{ i =1}^{\lfloor \frac{n}{2} \rfloor} q(Z_{\pi(i)}, Z_{\pi(i+\lfloor \frac{n}{2} \rfloor)}) - \mathbb{E}_{Z,Z'}[q(Z,Z')] \Big\|.
\end{align*}
Then according to Lemma \ref{lemma6} and the assumption in Lemma \ref{lemma7}, we obtain that
\begin{align*}
&\Big\| \frac{1}{\lfloor \frac{n}{2} \rfloor } \sum_{ i =1}^{\lfloor \frac{n}{2} \rfloor} q(Z_{\pi(i)}, Z_{\pi(i+\lfloor \frac{n}{2} \rfloor)}) - \mathbb{E}_{Z,Z'}[q(Z,Z')] \Big\|\\
\leq &\sqrt{\frac{2\sigma^2 \log \frac{2}{\delta}}{\lfloor \frac{n}{2} \rfloor}} + \frac{b \log \frac{2}{\delta}}{\lfloor \frac{n}{2} \rfloor}.
\end{align*}
The proof is complete.
\end{proof}

\begin{lemma}\cite{wainwright2019high}\label{lemma8}
If $Z_1,...,Z_n$ are sub-exponential random variables, then the classical Bernstein's inequality (see Theorem 2.10 in \cite{boucheron2013concentration}) holds with 
\begin{align*}
\sigma^2 = \frac{1}{n} \sum_{i=1}^{n} \| Z_i \|^2_{Orlicz-1}, \quad b = \max_{1 \leq i \leq n} \| Z_i \|_{Orlicz-1}.
\end{align*}
\end{lemma}




\ifCLASSOPTIONcaptionsoff
  \newpage
\fi



%


\bibliographystyle{abbrv}
\bibliography{ICML2020}
\end{document}